\documentclass[reqno,10pt]{article}
\usepackage[bookmarks=true,linkcolor=red,citecolor=blue,urlcolor=green,colorlinks=true]{hyperref}
\usepackage{amsmath}
\usepackage{mathrsfs}
\usepackage{amssymb}
\usepackage{enumerate}
\usepackage{setspace}
\usepackage{url}
\usepackage{amsmath}
\usepackage{amsfonts}
\usepackage{amssymb}
\usepackage{lscape}
\usepackage{booktabs}
\usepackage{bigstrut}
\usepackage{listings}

\usepackage{algorithm}
\usepackage[noend]{algorithmic}
\usepackage{multirow,rotating}
\usepackage{amssymb,amsmath}
\usepackage{booktabs}
\usepackage{graphicx}
\usepackage{makeidx}  
\usepackage{multirow}
\usepackage{enumerate}
\usepackage{verbatim}
\usepackage{theorem}
\usepackage{caption}
\usepackage{ifpdf}
\usepackage{bbm}
\usepackage{mathtools}
\usepackage{array}
\usepackage{subcaption}
\usepackage[a4paper, total={6.8in, 9in}]{geometry}

\newtheorem{theorem}{Theorem}[section]

\newtheorem{proposition}[theorem]{Proposition}

\newenvironment{proof}
 {{\sl Proof.}\hspace*{1 ex}}%
 {{\nopagebreak\hspace*{\fill}$\Box$\par\vspace{12pt}}}

\DeclarePairedDelimiterX{\inner}[2]{\langle}{\rangle}{#1, #2}

\newcommand{\Z}{\mathbb{Z}}
\newcommand{\R}{\mathbb{R}}

\newcommand{\ceil}[1]{{\left\lceil{#1}\right\rceil}}
\newcommand{\floor}[1]{{\left\lfloor{#1}\right\rfloor}}
\newcommand{\rint}[1]{{\left\lfloor{#1}\right\rceil}}

\allowdisplaybreaks
\begin{document}

\begin{center}
{\bf\Large On the implementation of a global optimization method for mixed-variable problems}
\end{center}
\vspace{5pt}
\begin{center}
Giacomo Nannicini \\
{\textit{{\small IBM Quantum, IBM T.J.~Watson research center, Yorktown Heights, NY}}}
\end{center}

\begin{abstract}
  We describe the optimization algorithm implemented in the
  open-source derivative-free solver RBFOpt. The algorithm is based on
  the radial basis function method of Gutmann and the metric
  stochastic response surface method of Regis and Shoemaker. We
  propose several modifications aimed at generalizing and improving
  these two algorithms: (i) the use of an extended space to represent
  categorical variables in unary encoding; (ii) a refinement phase to
  locally improve a candidate solution; (iii) interpolation models
  without the unisolvence condition, to both help deal with
  categorical variables, and initiate the optimization before a
  uniquely determined model is possible; (iv) a master-worker
  framework to allow asynchronous objective function evaluations in
  parallel. Numerical experiments show the effectiveness of these
  ideas.
\end{abstract}

\section{Introduction}
An optimization problem without any structural information on the
objective function or the constraints, but for which we have the
ability to evaluate them at given points, is called a {\em black-box}
problem. The area of {\em derivative-free optimization} is dedicated
to the study of optimization algorithms that do not rely on computing
the partial derivatives of the objective function, and it is naturally
applied to black-box problems. Many optimization problems in
engineering are solved by treating them as a black box, for two main
reasons: first, the objective function may not be known in an explicit
form, e.g., when it is the output of a complex computer simulation;
second, even if derivatives may exist and be computable, the effort
required may make it impractical, or the low accuracy of their
computation may make them unreliable.

This paper discusses the implementation of a global derivative-free
optimization algorithm that is specifically aimed at black-box
problems with expensive objective function evaluations. The algorithm
accepts as input problems of this form:
\begin{equation}
  \label{eq:problem}
  \left. \begin{array}{rrcl}
    \min & f(x, w) && \\
    & x & \in & [x^L, x^U] \\
    & x & \in & \R^{n_r} \times \Z^{n_d} \\
    & w & \in & \bigtimes\limits_{h=1}^{n_c} S_h,
    \end{array}
  \right\}
\end{equation}
where for $h=1,\dots,n_c$, $S_h$ is an (unordered) finite set, and
$x^L, x^U \in \R^{n_r + n_d}$ are vectors of finite lower and upper
bounds. Problem \eqref{eq:problem} is a mixed-variable problem,
because it involves continuous variables, discrete integer variables,
and discrete categorical variables (i.e., variables that belong to a
finite, not necessarily ordered set). Mixed-variable problems are
notoriously difficult to solve, but they have wide-ranging
applications, see e.g., \cite{audet2001pattern}. Note that
\eqref{eq:problem} is unconstrained: while constraints greatly
increase modeling capabilities, the majority of the derivative-free
optimization literature deals with unconstrained problems. This is
mainly due to two reasons: first, simple constraints can be
incorporated by penalizing their violation in the objective function;
second, problem \eqref{eq:problem} is already difficult to solve, so
unless the constraints are relatively easy to handle, complicated
(possibly black-box) constraints may make its solution too difficult
in practice\footnote{For an example of a difficult constrained
  black-box optimization problem, we refer the interested reader to
  the MOPTA 08 problem discussed at
  \url{https://www.miguelanjos.com/jones-benchmark}; see also
  \cite{regis2014constrained}.}.

Among the numerous methodologies proposed for derivative-free
optimization, there is an emerging consensus that algorithms based on
surrogate models typically have better global performance on nonconvex
problems with continuous variables. A {\em surrogate model} is a model
of the unknown objective function, that can be used by the
optimization algorithm as a proxy to obtain estimates of the objective
function value at unseen points in the domain. The algorithm discussed
in this paper employs a surrogate model constructed as a weighted
combination of radial basis functions (RBFs), plus a polynomial
tail. At each iteration, the algorithm uses the surrogate model to
determine the next point at which the objective function should be
evaluated; this decision is based on criteria first introduced in
\cite{gutmann01,regis07stochastic}, together with the modifications
discussed in \cite{nannirbfopt}. We generalize these approaches in
multiple ways, the most notable of which are:
\begin{enumerate}[(i)]
\item We introduce a surrogate model defined in an {\em extended space},
  mapping categorical variables to their unary encoding, and showing
  that all steps of the optimization algorithm can be performed in a
  natural way in either the original or the extended space.
\item We employ a periodic {\em refinement phase}, aimed at improving
  the best known solution with a local search. The local search
  consists of a small number of steps of an iterative gradient descent
  method, based on a linear local model of the objective function.
\item We do not enforce the {\em unisolvence condition} of the
  surrogate models, to both help with categorical variables (for
  reasons that will be discussed later), and to be able to start
  optimizing before a unique surrogate model is available.
\item We describe a {\em parallel implementation} of the algorithm
  that allows asynchronous, simultaneous objective function
  evaluations.
\end{enumerate}
The resulting optimization algorithm is implemented in an open-source
library called RBFOpt, first introduced in \cite{nannirbfopt}. The
paper \cite{nannirbfopt} is based on RBFOpt version 1.0, while this
paper discusses innovations introduced between version 1.0 and the
current 4.2 version\footnote{\cite{nannirbfopt} is based on Gutmann's
RBF method \cite{gutmann01}, which was the default global search
method in RBFOpt before being replaced in version 1.2.}. We give a
full description of several important implementation details that were
not previously discussed. Numerical experiments show the effectiveness
of these modifications on a set of nonconvex problems, as compared to
the algorithm of \cite{nannirbfopt}, and as compared to several
open-source derivative-free optimization solvers: NOMAD \cite{nomad},
Nevergrad \cite{nevergrad}, Optuna\cite{optuna2019}, Scikit-Optimize
\cite{scikit-optimize}, SMAC \cite{hutter11smac}. We provide an
example of a typical application by evaluating the performance of
RBFOpt for the optimization of the hyperparameters of a random forest
classifier on a given dataset. We remark that RBFOpt is designed for
deterministic black-box optimization problems, rather than
hyperparameter optimization problems where the result of each
objective evaluation is typically a sample from a random variable;
however, we can use RBFOpt by simply fixing the dataset and the random
seed used to train the classifier, thereby making the objective
function deterministic. This runs the risk of overfitting, as RBFOpt
only observes one realization of a generalization error estimator, but
in practice it can be an acceptable tradeoff. Results show that the
main innovations discussed above have a significant impact on
performance not only on artificial test functions, but also in this
specific hyperparameter optimization application\footnote{RBFOpt is
used in a commercial product to optimize hyperparameters of machine
learning models; while our benchmark set contains artificially
generated functions, development is largely driven by hyperparameter
optimization applications.}.

The rest of this paper is organized as follows. In
Sect.~\ref{s:surrogate} we review RBF interpolation. In
Sect.~\ref{s:extspace} we discuss two natural approaches to
incorporate categorical variables into the surrogate model, setting
the stage for an optimization algorithm. Sect.~\ref{s:algorithm}
describes the optimization algorithm, including several of the main
contributions of this paper. Finally, Sect.~\ref{s:exp}
provides an extensive numerical evaluation of the optimization
algorithm, and Sect.~\ref{s:conclusion} concludes the paper.

\section{Surrogate models with radial basis functions}
\label{s:surrogate}
Given $k$ distinct points $x^1, \dots, x^k \in \R^n$, a RBF
interpolant $s_k$ to the points $x^1, \dots, x^k \in \R^n$ is defined
as:
\begin{equation}
  \label{eq:s_k}
  s_k(x) := \sum_{i = 1}^k \lambda_i \phi(\|x - x^i\|) + p(x),
\end{equation}
where $\phi: \R_+ \to \R$, $\lambda_1,\dots,\lambda_k \in \R$ and $p$
is a polynomial of degree $d$. We use subscripts to refer to elements
of a vector, and superscripts to denote distinct vectors, e.g.,
$x^i_j$ is the $j$-the element of the $i$-th vector of the collection
$x^1,\dots,x^k$; the superscripts should not be confused with
exponents, because except for Table \ref{tab:rbf_degree}, all
polynomials in the rest of the paper do not involve variables --- only
scalars. Notice that here and in the rest of this paper, for
notational convenience we use $n$ as a general shorthand for the
dimension of the space in which the interpolation points live; the
value for $n$ is specified in the next section.  Furthermore, we
remark that here $x$ is a generic variable name, and should not be
intended to refer only to continuous and integer variables as in
\eqref{eq:problem}. The degree $d$ of the polynomial is chosen
according to Table \ref{tab:rbf_degree}, depending on the type of
radial basis functions $\phi(r)$.
\begin{table}[tb]
\begin{center} 
\begin{tabular}{l l r}
\hline 
$\phi(r)$ && $d$ \\ \hline
$r$ & (linear) & 0 \\
$r^3$ &(cubic) & 1 \\
$\sqrt{r^2+\gamma^2}$ &(multiquadric)&0 \\
$r^2\log{r}$&(thin plate spline)&1 \\
$e^{-\gamma r^2}$ &(Gaussian)&-1 \\
\hline
\end{tabular}
\end{center}
\caption{RBF functions available in RBFOpt.}
\label{tab:rbf_degree}
\end{table}

If $\phi(r)$ is cubic or thin plate spline, we obtain an interpolant
of the form:
\begin{equation}
  s_k(x) := \sum_{i = 1}^k \lambda_i \phi(\|x - x^i\|) +
  \alpha^{\top} \begin{pmatrix} x \\ 1 \end{pmatrix},
\end{equation}
where $\alpha \in \R^{n+1}$. The values of $\lambda_i, \alpha$ can be determined
by solving the following linear system:
\begin{equation}
  \label{eq:phisystem}
  \begin{pmatrix} \Phi & P \\ P^{\top} & 0_{(n+1)\times
      (n+1)} \end{pmatrix} \begin{pmatrix} 
    \lambda \\ \alpha \end{pmatrix} = \begin{pmatrix} F \\ 0_{n+1} \end{pmatrix},
\end{equation}
with:
\begin{equation*}
  \Phi = \left(\phi(\| x^i - x^j \|)\right)_{i,j=1,\dots,k},
  \quad
  P = \begin{pmatrix} (x^1)^{\top} & 1 \\ \vdots & \vdots \\ (x^k)^{\top} &
    1 \end{pmatrix},
  \quad
  \lambda = \begin{pmatrix} \lambda_1 \\ \vdots\\ \lambda_k 
  \end{pmatrix},
  \quad
  F = \begin{pmatrix} f(x^1) \\ \vdots \\ f(x^k) \end{pmatrix}.
\end{equation*}
If $k \ge n+1$, $\text{rank}(P)= n + 1$, and the points
$x^1,\dots,x^k$ are pairwise distinct, then \eqref{eq:phisystem} is
nonsingular; this is a sufficient but not necessary condition, used by
Gutmann's RBF algorithm \cite{gutmann01} to guarantee uniqueness of the
interpolant on problems with continuous variables.

If $\phi(r)$ is linear or multiquadric, $d=0$ and the system
\eqref{eq:phisystem} has a simpler expression: $P$ is the all-one
column vector of dimension $k$.  In the Gaussian case, $d = -1$ and
$P$ is removed from system (\ref{eq:phisystem}).  The dimensions of
the zero matrix and vector in (\ref{eq:phisystem}) are adjusted
accordingly.

In the setting of this paper, the matrix of \eqref{eq:phisystem} may
be singular in some situations (see Sect.~\ref{s:lssol}), hence this
assumption no longer holds; however, we generally strive to obtain a
nonsingular linear system so that the coefficients of $s_k$ can be
uniquely determined.

\section{Optimization with categorical variables in extended space}
\label{s:extspace}
In the rest of this paper, we assume that $x^L_j, x^U_j \in \Z$ for
all $j = n_r+1,\dots,n_r+n_d$. For $i=1,\dots,n_c$, we define $m_h =
|S_h|$ and $\hat{m}_h = \sum_{k=1}^{h} m_k$, with $\hat{m}_0 = 0$ for
convenience. We use $\inner{\cdot}{\cdot}$ to denote inner
products. Since \eqref{eq:problem} has categorical variables, which
are difficult to handle in any mathematical optimization framework due
to their unstructured nature, we work with two inexact formulations
for \eqref{eq:problem}. The first formulation, which we call {\em
  original space} formulation, simply replaces the categorical
variables with integer variables. Define the vectors $x^{o,L}, x^{o,U}
\in \R^{n_r+n_d+n_c}$ as:
\begin{equation*}
  x^{o,L}_i = \begin{cases}
    x^L_i & \text{if } i \le n_r + n_d \\
    1 & \text{otherwise}
  \end{cases}
  \qquad
  x^{o,U}_i = \begin{cases}
    x^U_i & \text{if } i \le n_r + n_d \\
    m_{i-n_r-n_d} & \text{otherwise}.
  \end{cases}
\end{equation*}
Then the original space formulation is defined as:
\begin{equation}
  \label{eq:origspace}
  \left. \begin{array}{rrcl}
    \min & f(x, C(x_{n_r+n_d+1},\dots,x_{n_r+n_d+n_c})) && \\
    & x & \in & [x^{o,L}, x^{o,U}] \\
    & x & \in & \R^{n_r} \times \Z^{n_d + n_c} \\
    \end{array}
  \right\}
\end{equation}
where $C : \bigtimes_{h=1}^{n_c} [1, \dots, m_h] \to \bigtimes_{h=1}^{n_c}
S_h$ is a one-to-one map of the integers $[1, \dots, m_h]$ to elements of the
$m_h$-dimensional set $S_h$. Notice that to
construct the function $C$, we must arbitrarily define an order of
each set $S_h$. This allows us to apply any algorithm for
mixed-integer black-box problems directly to
\eqref{eq:problem}. However, it is an inherently flawed approach,
because the sets $S_h$ are originally unordered. Since virtually all
derivative-free optimization algorithms use metric information, we are
imposing on the problem artificial structure that is not reflected in
its original formulation.

The second formulation, which we call {\em extended space}
formulation, uses a unary encoding for the categorical variables. Define the vectors $x^{e,L},
x^{e,U} \in \R^{n_r+n_d+\hat{m}_{n_c}}$ as:
\begin{equation*}
  x^{e,L}_i = \begin{cases}
    x^L_i & \text{if } i \le n_r + n_d \\
    0 & \text{otherwise}
  \end{cases}
  \qquad
  x^{e,U}_i = \begin{cases}
    x^U_i & \text{if } i \le n_r + n_d \\
    1 & \text{otherwise}.
  \end{cases}
\end{equation*}
Then the extended space formulation is defined as:
\begin{equation}
  \label{eq:extspace}
  \left. \begin{array}{rrcl}
    \min & f(x, \hat{C}(x_{n_r+n_d+1},\dots,x_{n_r+n_d+\hat{m}_{n_c}})) && \\
    & x & \in & [x^{e,L}, x^{e,U}] \\
    & x & \in & \R^{n_r} \times \Z^{n_d} \times \{0,1\}^{\hat{m}_{n_c}} \\
    \forall h=1,\dots,n_c & \sum_{j=\hat{m}_{h-1}+1}^{\hat{m}_h} x_{n_r+n_d+j} & = & 1.
    \end{array}
  \right\}
\end{equation}
where $\hat{C} : \{0,1\}^{\hat{m}_{n_c}} \to \bigtimes_{h=1}^{n_c}
S_h$ maps the binary vector
$(x_{n_r+n_d+1},\dots,x_{n_r+n_d+\hat{m}_{n_c}}) \in
\{0,1\}^{\hat{m}_{n_c}}$ to a choice of elements from the sets $S_h$,
by viewing it as the juxtaposition of the characteristic vectors of
the sets $S_h$. This mapping assigns one value to each categorical
variable, because of the constraints $\sum_{j=\hat{m}_{h-1} +
  1}^{\hat{m}_h} x_{n_r+n_d+j} = 1$ for all $h=1,\dots,n_c$. In the following we
denote the feasible region of \eqref{eq:extspace} as
$\Omega_e$. Notice that similar to the previous formulation,
\eqref{eq:extspace} also suffers from the flaw of imposing an order on
the sets $S_h$; however, we show next that a surrogate model of
\eqref{eq:extspace} with radial basis functions ignores the order,
thereby avoiding ranking points based on artificial metric information
(i.e., that does not exist in the original problem).
\begin{proposition}[Invariance to permutation]
  \label{prop:s_k_extspace}
  Let $x^1,\dots,x^k \in \Omega_e \subset \R^{n_r+n_d} \times
  \{0,1\}^{\hat{m}_{n_c}}$ with corresponding function values
  $y^1,\dots,y^k$. Let $\pi = \bigtimes_{h=1}^{n_c} \pi^h$ be a
  permutation of $\{0,1\}^{\hat{m}_{n_c}}$, where for $h=1,\dots,n_c$,
  $\pi^h$ is a permutation of $\{0,1\}^{m_h}$. Let $\pi^e :=
  I_{n_r+n_d} \times \pi$ be the extension of $\pi$ to an operator on
  $(n_r+n_d+\hat{m}_{n_c})$-dimensional vectors that acts as the
  identity on the first $n_r+n_d$ components. Let $\lambda \in \R^k,
  \alpha \in \R^{n_r+n_d+\hat{m}_{n_c}}, \alpha^0 \in \R$ define an interpolant
  \begin{equation}
    \label{eq:s_k_extspace}
    s_k(x) := \sum_{i = 1}^k \lambda_i \phi(\|x - x^i\|) + \inner{\alpha}{x} + \alpha^0
  \end{equation}
  to the points $x^1,\dots,x^k$ with values
  $y^1,\dots,y^k$. (If the polynomial tail is of degree $0$ according
  to Table~\ref{tab:rbf_degree}, then $\alpha$ is the all-zero
  vector; if the degree is $-1$, $\alpha^0$ is 0 as well.) Then for any
  $x \in \R^{n_r+n_d} \times \{0,1\}^{\hat{m}_{n_c}}$, the
  function $s'_k$, defined as:
  \begin{equation*}
    s'_k(x) := \sum_{i = 1}^k \lambda_i \phi(\|x - \pi^e(x^i)\|) + \inner{\pi^e(\alpha)}{x} + \alpha^0,
  \end{equation*}
  is such that $s_k(x) = s'_k(\pi^e(x))$.
\end{proposition}
\begin{proof}
  Since by definition $\pi^e$ is a permutation of the components of the
  vector $(x_{n_r+n_d+1},\dots,\allowbreak x_{n_r+n_d+\hat{m}_{n_c}})$, and acts as the identity on the first $n_r+n_d$ components, we have:
  \begin{equation*}
   \|\pi^e(x) - \pi^e(x^i)\| = \|x - x^i\|
  \end{equation*}
  and
  \begin{equation*}
    \inner{\pi^e(\alpha)}{\pi^e(x)} = \inner{\alpha}{x}.
  \end{equation*}
  This immediately implies $s_k(x) = s'_k(\pi^e(x))$.
\end{proof}
Prop.~\ref{prop:s_k_extspace} implies that the surrogate model in
extended space is invariant to the order adopted in the unary-encoding
representation of the categorical variables. Indeed, if the solution
to \eqref{eq:phisystem} is unique, yielding a unique surrogate model
$s_k$, after permuting the unary encoding of the categorical variables
we would obtain the same surrogate model from \eqref{eq:phisystem}. We
remark that if the solution to \eqref{eq:phisystem} is not unique (see
Sect.~\ref{s:lssol}), then one may obtain a different $s_k$ after
permuting the unary encoding of the categorical variables; however,
even in this case, each solution to \eqref{eq:phisystem} has an
equivalent solution for the system obtained after permutation. Note
that similar properties do not hold when using the original space
formulation: if categorical variables are represented by integers in
the interval $[1, \dots, m_h]$, permuting these integers is not a
component-wise permutation of the vector $x$, and could in general
change the norms of $\|x - x^i\|$.

We can therefore use the extended space formulation
\eqref{eq:extspace} together with surrogate models of the form
\eqref{eq:s_k} to ensure that the sets $S_i$ are correctly treated as
unordered. However, \eqref{eq:extspace} is a constrained formulation,
whereas the algorithms of \cite{gutmann01} and
\cite{regis07stochastic} (that RBFOpt is based on) assume
unconstrained problems. In the next sections we describe one way to
deal with the constraints in \eqref{eq:extspace}. Another difficulty
is given by the fact that the constraints lead to linearly dependent
columns in the submatrix $P$ of \eqref{eq:phisystem}; this issue is
also discussed in the next sections. From now on, we define $n := n_r
+ n_d + \hat{m}_{n_c}$, i.e., the dimension of the extended space: the
interpolation model \eqref{eq:s_k} lives in $n$-dimensional
space. Note that we always require the representation of the
categorical variables in extended space to take on integer values,
as is natural. This is in contrast with popular hyperparameter
optimization approaches, where the categorical variables are often
treated as continuous for simplicity, and the fractional values are
then mapped to a valid discrete value in some heuristic way (such as
setting the variable with the largest fractional value to 1, and the
rest to 0; this is the approach implemented, e.g., in Spearmint
\cite{snoek2012practical}). From an optimization standpoint, it seems
more rigorous to treat integer variables as such, because it is
well-known that a solution to the relaxed problem could be very far
from the integer optimum, even for linear optimization problems
\cite{papadimitriou}.

Finally, to better understand the structure of the surrogate model
$s_k$ in the extended space, we rewrite it as follows. For $x \in
\R^{n_r} \times \Z^{n_d} \times \{0,1\}^{\hat{m}_{n_c}}$, define
$\text{cat}(x, h) :=
(x_{n_r+n_d+\hat{m}_{h-1}+1},\dots,x_{n_r+n_d+\hat{m}_{h}})$, i.e.,
the subvector corresponding to the unary representation of the $h$-th
categorical variable. With this definition, note that $s_k$ can be
rewritten as:
\begin{equation*}
  s_k(x) = \sum_{i = 1}^k \lambda_i \phi(\|(x_1,\dots,x_{n_r+n_d}) - (x^i_1,\dots,x^i_{n_r+n_d}) +
  \sum_{h=1}^{n_c} \mathbbm{1}_{\text{cat}(x,h) \neq \text{cat}(x^i,h)} \|) + \inner{\alpha}{x} + \alpha^0.
\end{equation*}
From the above equation we can see that the argument of the radial
basis function centered at the interpolation point $x^i$ is shifted by
the number of categorical variables that disagree with $x^i$. Thus,
for the radial basis function part of the interpolant, the surrogate
model is determined by the non-categorical variables, as well as the
number of disagreements with the categorical variables at the
interpolation nodes: the only notion of distance between categorical
variables is reduced to the binary information agreement/disagreement,
which is independent from the order assigned to the sets
$S_h$. Furthermore, depending on the degree of the polynomial tail,
there can be an additional shift of the entire surrogate model that
depends only on the values of the categorical variables (i.e., the
part corresponding to categorical variables in the inner product term
$\inner{\alpha}{x}$).

It should be noted that if a categorical variable, say the first
categorical variable for simplicity, has only two possible values,
i.e., $|S_1| = 2$, then the extended space formulation is redundant:
the constraint $x_{n_r+n_d+1} + x_{n_r+n_d+2} = 1$ implies that
$x_{n_r+n_d+2}$ is simply the complement of $x_{n_r+n_d+1}$. As will
be discussed in Sect.~\ref{s:lssol}, one among $x_{n_r+n_d+1},
x_{n_r+n_d+2}$ would always be eliminated when determining the
coefficients of the surrogate model. Hence, we use the extended space
formulation only for categorical variables that have strictly more
than two possible values: for those that have exactly two, we use the
original space formulation, mapping them to a binary variable.

\section{Description of the optimization algorithm}
\label{s:algorithm}
Many RBF-based global optimization methods use a similar scheme
that attempts to balance {\em exploration} (trying to improve a
surrogate model of the objective function in unknown parts of the
domain) with {\em exploitation} (trying to find the best objective
function value based on the current surrogate model); see, e.g.,
\cite{gutmann01,regis07stochastic,muller13somi,muller15miso,pysot}. The
algorithm that we propose is no exception, although we introduce some
additional steps (Refinement step and Restoration step, see below) as
compared to the more traditional framework. More specifically, we use
the following optimization scheme:
\begin{itemize}
\item {\bf Initial step:} Set $k$ equal to the size of the initial
  sample set. Choose $k$ affinely independent points $x^1, \dots,
  x^{k} \in \Omega_e$ using an initialization strategy.

\item {\bf Iteration step:} Repeat the following steps until $k$
  exceeds the prescribed number of function evaluations.
  \begin{enumerate}[(i)]
  \item Compute the RBF interpolant $s_k$ to the points $x^1, \dots,
    x^k$, solving \eqref{eq:phisystem}. If the system is not full
    rank, find the least squares solution. If the system cannot be
    solved, go to Restoration step.
  \item Choose a trade-off between {\em exploration} and {\em exploitation}.
  \item Determine the next point $x^{k+1}$ based on the choice at step
    (ii).
  \item Evaluate $f$ at $x^{k+1}$.
  \item Set $k \leftarrow k+1$. If the last Refinement step was
    performed sufficiently many iterations ago, go to the Refinement
    step. Otherwise, repeat the Iteration step.
  \end{enumerate}
\item {\bf Refinement step:} 
  \begin{enumerate}[(i)]
  \item Select $n+1$ points out of $x^1, \dots, x^k$ to initialize a
    local model.
  \item Apply a local search method for a specified number $k'$ of
    iterations, obtaining points $x^{k+1}, \dots, x^{k+k'}$.
  \item Set $k \leftarrow k + k'$ and go back to the Iteration step.
  \end{enumerate}
\item {\bf Restoration step:} Attempt to change the set of
  interpolation points so that \eqref{eq:phisystem} admits a
  solution. If successful, return to Iteration step. Otherwise,
  restart the algorithm.
\end{itemize}
The algorithm described above can be considered a meta-algorithm, with
many possible instantiations. The choice of the initial sample points
is discussed in \cite{nannirbfopt}; in this paper we always select
them by constructing a latin hypercube design aimed at maximizing the
minimum distance between the sample points. In the following, we
provide an overview of the main different implementations of the above
meta-algorithm available in RBFOpt. We remark that \cite{nannirbfopt}
describes several improvements to the meta-algorithm (in the context
of Gutmann's RBF method \cite{gutmann01}); all of them are used by
default in RBFOpt. Most notably, these are: automatic scaling of the
domain of the function; clipping and rescaling of the codomain;
restriction of the search box during global search --- see
\cite{nannirbfopt} for details.

\subsection{Solution of linear systems and non-unique interpolants}
\label{s:lssol}
To compute the surrogate model $s_k$ we must solve system
\eqref{eq:phisystem}. However, when the polynomial $p(x)$ is of degree
1, if some of the interpolation points are affinely dependent then
\eqref{eq:phisystem} has determinant 0. In continuous space, the
algorithm never generates affinely dependent points\footnote{To be
  precise, the algorithm only guarantees pairwise distinct points; but
  the probability of selecting a new point that is affinely spanned by
  the previous points is $0$ with the MSRSM algorithm
  \cite{regis07stochastic}, and only happens in ill-conditioned cases
  for Gutmann's algorithm \cite{gutmann01}.}. With categorical
variables, this is bound to happen: because of the constraints
$\sum_{j=\hat{m}_{h-1}}^{\hat{m}_h} x_{n_r+n_d+j} = 1$ for all
$h=1,\dots,n_c$, the binary representation of each categorical
variable in extended space adds up to the all-one vector, which is
already a column of \eqref{eq:phisystem} whenever $d=1$ in
Table~\ref{tab:rbf_degree}. To solve this issue, whenever the problem
has categorical variables and $d=1$, we eliminate the columns
$x_{n_r+n_d+\hat{m}_h}$ for $h=1,\dots,n_c$ and the corresponding rows
from \eqref{eq:phisystem}. These are precisely the last columns of
each constraint $\sum_{j=\hat{m}_{h-1}}^{\hat{m}_h} x_{n_r+n_d+j} = 1$
for all $h=1,\dots,n_c$. This is motivated by the following simple
observation.
\begin{proposition}[Reduced linear system]
  Suppose $d=1$ and $n_c \ge 1$, i.e., there is at least one
  categorical variable. Suppose further that we employ the extended
  space formulation of the problem \eqref{eq:extspace}. Denote by
  $\hat{P}$ the matrix obtained by eliminating the columns
  $x_{n_r+n_d+\hat{m}_h}$ for $h=1,\dots,n_c$ from $P$. Then if the system
  \eqref{eq:phisystem} has a solution, so does the system:
  \begin{equation}
    \label{eq:rphisystem}
  \begin{pmatrix} \Phi & \hat{P} \\ \hat{P}^{\top} & 0_{(n+1-n_c)\times
      (n+1-n_c)} \end{pmatrix} \begin{pmatrix} 
    \lambda \\ \alpha \end{pmatrix} = \begin{pmatrix} F \\ 0_{n+1-n_c} \end{pmatrix},
  \end{equation}
\end{proposition}
\begin{proof}
  Let $\bar{\lambda}, \bar{\alpha}$ be a solution to
  \eqref{eq:phisystem}. Since $x_{n_r+n_d+\hat{m}_h} =
  1-\sum_{j=\hat{m}_{h-1}}^{\hat{m}_h-1} x_{n_r+n_d+j}$ for all
  $h=1,\dots,n_c$, we can eliminate $x_{n_r+n_d+\hat{m}_1}$ from $P$; if we define $v = (1,\dots,1,-1)^{\top} \in \R^{n+1}$ as the vector with $-1$ in the last component and $1$ in all other components, the substitution yields:
  \begin{equation*}
  \begin{pmatrix} \Phi & P \\ P^{\top} & 0_{(n+1)\times
      (n+1)} \end{pmatrix} \begin{pmatrix} \bar{\lambda}
    \\ \bar{\alpha} - \bar{\alpha}_{\hat{m}_1} v\end{pmatrix}
    = \begin{pmatrix} F \\ 0_{n+1} \end{pmatrix}.
  \end{equation*}
  This shows that $(\bar{\lambda}, \bar{\alpha} -
  \bar{\alpha}_{\hat{m}_1} v)$ is also a solution to
  \eqref{eq:phisystem}. However, by definition the
  $\hat{m}_1$-component of $\bar{\alpha} - \bar{\alpha}_{\hat{m}_1} v$
  is zero, implying that we can eliminate the column corresponding to
  $x_{n_r+n_d+\hat{m}_1}$ from $P$ (this also eliminates one row from
  $P^{\top}$, which obviously does not restrict the set of solutions
  to the system). We can repeat this process for
  $x_{n_r+n_d+\hat{m}_h}$ for $h=2,\dots,n_c$, showing that the
  reduced system admits a solution and completing the proof.
\end{proof}
By the above proposition, we can solve \eqref{eq:rphisystem} rather
than \eqref{eq:phisystem}, find a solution to the smaller system, and
extend it to a full solution by inserting zeroes in the positions
corresponding to eliminated columns. The advantage of this approach is
that \eqref{eq:rphisystem} may be an invertible system whereas
\eqref{eq:phisystem} is not invertible under the stated conditions.

Affinely dependent points affect not only the nonsingularity of the
system \eqref{eq:phisystem}, but also the unisolvence property of RBF
interpolants, i.e., uniqueness of the interpolant
\cite{powell2005five}. In particular, when $d=1$ in
Table~\ref{tab:rbf_degree}, the sufficient condition for unisolvence
--- using a basis of polynomials of degree 1 --- fails because we
eliminate one or more monomials from the polynomial basis. Thus, when
$d=1$ we can no longer guarantee the unisolvence property. However, in
practice we observe that the system \eqref{eq:phisystem} often has a
solution even when this condition fails, and sometimes a unique
solution; this was also observed in \cite{du2008radial}.

Even when using the reduced matrix $\hat{P}$, it can sometimes happen
that the algorithm generates affinely dependent interpolation
points. Specifically, this can occur when there are integer or
categorical variables, where column entries belong to a discrete set;
empirically, we observe this especially when the problem has many
binary variables. When this happens, we solve \eqref{eq:rphisystem} as
a least-squared-residuals problem. This is computationally more
expensive, but guarantees a solution. (The time spent in the solution
of linear systems is negligible in practice.)

The least squares solution to the linear system is also used whenever
there are not enough sample points to build a full interpolant, i.e.,
$k \le n + 1$. Whereas the majority of the literature assumes that at
least $n+1$ points are sampled in the initialization phase (see, e.g.,
\cite{gutmann01,regis07stochastic,nannirbfopt}), in practice this can
be a severe drawback when $n$ is large. Approaches to begin the
optimization before sampling $n+1$ points are described in
\cite{regis2013initialization,sartor2017thesis}; we follow the
approach of \cite{sartor2017thesis}\footnote{The numerical tests in
  \cite{sartor2017thesis} are based on a customized version of RBFOpt.}.
Specifically, the number of initial sample points $n_{\text{init}}$ is
heuristically chosen according to the following formula:
\begin{equation}
  \label{eq:initformula}
  n_{\text{init}} = \begin{cases}
    \rint{0.5(n+1)} & \text{if } n \le 20 \\
    \rint{0.4(n+1)} & \text{otherwise. } 
  \end{cases}
\end{equation}
If RBFOpt is executed in parallel with at least 2 threads, then the
number of initial sample points is chosen as:
\begin{equation*}
  n_{\text{init}} = \begin{cases}
    n+1 & \text{if } n \le 20 \\
    \rint{0.75(n+1)} & 21 \le n \le 50 \\
    \rint{0.5(n+1)} & \text{otherwise. } 
  \end{cases}
\end{equation*}
As long as $k \le n$, we use the least squares solution to determine
the coefficients of the surrogate model $s_k$; the rest of the
optimization algorithm remains unchanged. Whenever $k \ge n+1$ points
are available and they are affinely independent, system
\eqref{eq:phisystem} has a unique solution and we compute it using a
direct method.

The reduced matrix $\hat{P}$ is also employed in the Initial step of
the algorithm. After generating an initial sample set (see
\cite{nannirbfopt} for a description of the strategies to do so
implemented in RBFOpt), we compute a singular value decomposition of
$\hat{P}$; as long as some singular value is close to zero, we
generate a new sample set. Notice that if there are no categorical
variables then $\hat{P}$ coincides with $P$.

We remark that for all RBFs that do not have a polynomial tail of
degree 1, i.e., all except the cubic and thin plate splines, these
additional steps are not necessary. However, the cubic and thin plate
spline RBFs are empirically among the most accurate, see
e.g.~\cite{nannirbfopt}, and the automatic model selection procedure
employed by RBFOpt (see Sect.~\ref{s:modelsel}) chooses one of these
two RBFs very often in practice. Hence, the additional effort is
justified.

\subsection{Determining the next point: Iteration step}
\label{s:iteration}
We implement a variation of two algorithms for global
optimization using RBFs: Gutmann's RBF algorithm \cite{gutmann01} and
the Metric Stochastic Response Surface Method (MSRSM)
\cite{regis07stochastic}. Both algorithms proceed in cycles, and use a
parameter $\kappa$ that determines the length of an optimization
cycle.

\subsubsection{Gutmann's RBF algorithm}
A detailed description is given in \cite{nannirbfopt}; here we report
the main steps only. Let $\ell_k$ be the RBF interpolant to the points
$\{x^i : i = 1,\dots,k\} \cup \{y\}$, with function values $0, 0,
\dots, 0, 1$ respectively. Let $\mu_k(y)$ be the coefficient of $\ell_k$
corresponding to the RBF centered at $y$. Define
\begin{equation*}
  g_k(y) = (-1)^{d+1} \mu_k(y)[s_k(y) - f_k^\ast]^2, \qquad y
    \in \Omega_e \setminus\{x^1, \dots, x^k\},
\end{equation*}
where $f_k^\ast$ is a given value. Furthermore, define:
\begin{equation}
  \label{eq:h_k}
  h_k(x) = \begin{cases} \frac{1}{g_k(x)} & \text{if } x \not\in
    \{x_1,\dots,x_k\} \\
    0 & \text{otherwise}.
  \end{cases}
\end{equation}
Gutmann's RBF method then implements the following Iteration step:
\begin{itemize}
\item {\bf Iteration step} (for Gutmann's RBF algorithm):
  \begin{itemize}
  \item[(ii)] Choose a target value $f^\ast_k \in \R \cup \{-\infty\} :
    f^\ast_k \le \min_{x \in \Omega_e} s_k(x)$.
  \item[(iii)] Compute 
    \begin{equation}
      \label{eq:gutmannnext}
      x_{k+1} = \arg \max_{x \in \Omega_e} h_k(x),
    \end{equation}
    where $h(x)$ is defined as in (\ref{eq:h_k}).
  \end{itemize}
\end{itemize}

Let $y^\ast := \arg \min_{x \in \Omega_e} s_k(x)$, $f_{\min}
:= \min_{i=1,\dots,k} f(x^i)$, and $f_{\max} := \max_{i=1,\dots,k}
f(x^i)$. We employ a cyclic strategy that picks target values
$f^\ast_k\in \R \cup \{-\infty\}$ according to the
following sequence of length $\kappa + 2$:
\begin{itemize}
\item Step $-1$ ({\it InfStep}): Choose $f^\ast_k \leftarrow -\infty$.
  In this case the problem of finding $x^{k+1}$ can be rewritten as:
  \begin{equation*}
    x_{k+1} = \arg \max_{x \in \Omega_e} \frac{1}{(-1)^{d+1}\mu_k(x)}.
  \end{equation*}
  This is a pure exploration phase, yielding a point far from
  $x_1,\dots,x_k$. 
  
\item Step $\ell \in \{0,\dots,\kappa-1\}$ ({\it Global search}): Choose
  \begin{equation}
    \label{eq:targetval}
    f^\ast_k \leftarrow s_k(y^\ast) - (1 - \ell/\kappa)^2(f_{\max} -
    s_k(y^\ast)). 
  \end{equation}
  In this case, we try to strike a balance between improving model
  quality and finding the minimum.

\item Step $\kappa$ ({\it Local search}): Choose $f^\ast_k \leftarrow
  s_k(y^\ast)$. Notice that in this case \eqref{eq:h_k} is maximized
  at $y^\ast$. Hence, if $s_k(y^\ast) < f_{\min} - 10^{-10}
  |f_{\min}|$ we accept $y^\ast$ as the new sample point $x_{k+1}$
  without solving \eqref{eq:gutmannnext}. Otherwise we choose
  $f^\ast_k \leftarrow f_{\min} - 10^{-2} |f_{\min}|$. This is an
  exploitation phase.
\end{itemize}

\subsubsection{MSRSM algorithm}
\label{s:msrsm}
Define $\text{dist}(x) := \min_{i=1,\dots,k} \|x - x^i\|$. The MSRSM
algorithm implements the following Iteration step:
\begin{itemize}
\item {\bf Iteration step} (for the MSRSM algorithm):
  \begin{itemize}
  \item[(ii)] Choose a target value $\alpha \in [0,1] \cup \{\infty\}$.
  \item[(iii)] Choose a finite set of reference points $R \subset
    \Omega_e \setminus \{x^1,\dots,x^k\}$, and compute
    \begin{equation}
      \label{eq:msrsmnext}
      x_{k+1} = \arg \min_{x \in \Omega_e} \alpha \frac{\max_{y \in R} \text{dist}(y) - \text{dist}(x)  }{ \max_{y \in R } \text{dist}(y) - \min_{y \in R } \text{dist}(y)} + \frac{s_k(x) - \min_{y \in R } s_k(y)}{ \max_{y \in R } s_k(y) -  \min_{y \in R } s_k(y)}. 
    \end{equation}
  \end{itemize}
\end{itemize}
Essentially, \eqref{eq:msrsmnext} tries to solve a bi-objective
optimization problem in which the two objective functions are the
(negative of the) maximin distance from the points $x^1,\dots,x^k$,
and the value of the surrogate model. The paper
\cite{regis07stochastic} uses a variation of \eqref{eq:msrsmnext}, in
which the second fraction in the expression has weight $(1-\alpha)$
rather than $1$. RBFOpt supports this version, but by default it uses
equation \eqref{eq:msrsmnext} instead (see also \cite{nanniirrbfopt}).

The value of $\alpha$ is chosen according to a cyclic strategy of
length $\kappa + 2$ in which each step has similar goals to the
corresponding step discussed in Gutmann's RBF method. The cyclic
strategy is as follows:
\begin{itemize}
\item Step $-1$ ({\it InfStep}): Choose $\alpha \leftarrow \infty$.
  In this case the problem of finding $x^{k+1}$ can be rewritten as:
  \begin{equation*}
    x_{k+1} = \arg \max_{x \in \Omega_e} \min_{i=1,\dots,k} \|x - x^i\|.
  \end{equation*}
  This is a pure exploration phase. 
  
\item Step $\ell \in \{0,\dots,\kappa-1\}$ ({\it Global search}): Choose
  $\alpha \leftarrow \max\{1 - (\ell+1)/\kappa, 0.05\}$. This aims for
  balance between exploration and exploitation.

\item Step $\kappa$ ({\it Local search}): Choose $\alpha \leftarrow
  0$. In this case, the solution to \eqref{eq:msrsmnext} is the point
  that minimizes the surrogate model, i.e., $y^\ast = \arg \min_{y \in
    \Omega_e} s_k{y}$. If $y^\ast$ is such that $s_k(y^\ast) <
  f_{\min} - 10^{-10} |f_{\min}|$, we accept $y^\ast$ as the new point
  $x_{k+1}$. Otherwise, choose $\alpha \leftarrow 0.05$. This is an
  exploitation phase.
\end{itemize}

\subsubsection{Solution of the search problems}
\label{s:subproblems}
We implement three different approaches for the solution of the
optimization problems \eqref{eq:gutmannnext} and
\eqref{eq:msrsmnext}:
\begin{enumerate}[(1)]
\item Problems \eqref{eq:gutmannnext} and \eqref{eq:msrsmnext} are
  solved with a simple genetic algorithm, that works by generating an
  initial population $X$ uniformly at random, then iteratively
  constructing a new population by taking:
  \begin{itemize}
  \item The $0.25|X|$ best points in $X$ (surviving population),
    according to the objective function being optimized;
  \item $0.25|X|$ points obtained by repeatedly performing the
    following procedure: we randomly pick two points $x^1, x^2$ from
    the surviving population, and create a new point by choosing each
    entry from either $x^1$ or $x^2$ (mating);
  \item $0.5|X|$ points generated uniformly at random (new individuals);
  \item a point obtained by taking the best individual in $X$, and
    randomly perturbing some of its entries (mutation). The number of
    mutated entries increases as the number of iterations of the
    genetic algorithm increases.
  \end{itemize}
  We appropriately round the above quantities so that the
  size of the population $|X|$ remains constant. In the presence of
  categorical variables we sample points in the original space, where
  uniform random sampling is easily implemented, then map them to the extended
  space.
\item Rather than solving \eqref{eq:gutmannnext} and
  \eqref{eq:msrsmnext} directly, we sample a large number of points in
  $\Omega_e$ and choose the best point in the sample. This is the
  approach advocated in \cite{regis07stochastic}. In the presence of
  categorical variables we sample points in the original space, where
  uniform random sampling is easily implemented, then map them to the
  extended space.
\item Problems \eqref{eq:gutmannnext} and \eqref{eq:msrsmnext} are
  solved by means of the mathematical programming solvers Ipopt and
  Bonmin. This is the approach advocated in \cite{gutmann01}. Since
  Bonmin supports constrained problems, we work directly in the
  extended space when this approach is chosen (note that we must use
  Bonmin if discrete variables are present).
\end{enumerate}
We remark that the MSRSM scoring function requires a set of reference
points $R$, see \eqref{eq:msrsmnext}: the set of reference points is
taken to be the current population for the genetic algorithm, the
whole sample when using the sampling scheme, and $x^1,\dots,x^k$ for
when using a mathematical programming solver.

\subsection{Determining the next point: Refinement step}
As indicated at the beginning of Sect.~\ref{s:algorithm}, during the
search we periodically execute a Refinement step, with the purpose of
improving the best solution available by performing a local search
around it. The scheme employed in the Refinement step is reminiscent
of a trust region method \cite{conn2000trust,wild13rbf}. However, it
is not a trust region method, mainly because we construct a local
model using points that may be outside the trust region, and we do not
require that the model is fully-linear or a similar property
\cite{conn2009global} (although the QR-like algorithm that we use to
improve the geometry of the interpolation set would in principle yield a
fully-linear model, if it were allowed to run to completion
\cite{wild13rbf,conn2008geometry}). Furthermore, our scheme is adapted
to work on mixed-variable problems, rather than only problems with
continuous variables; proving rigorous local convergence guarantees in
the discrete setting is an involved task in itself, see
e.g.,~\cite{liuzzi2020algorithmic}, and here we limit ourselves to a
heuristic approach to refine candidate solutions. While trust region
methods enjoy strong convergence properties \cite{conn2009global},
managing the set of sample points and converging to a stationary point
can be expensive, compared to surrogate model methods, in terms of
number of objective function evaluations. Empirically, we found that
embedding a full trust region method for local search could severely
slow the global search, which is the main strength of RBF-based
surrogate model methods; hence, we opted for the methodology described
below, that is guided by two design priciples: (1) it is initialized
using information from known points only; (2) it is quickly stopped if
it fails to yield any improvement. Note that with our approach the RBF
surrogate model is still used for global and local search, but it is
complemented by a local linear model to search around the best known
solution; this is contrast to the approach recently proposed in 
\cite{eriksson2019scalable}, where the global surrogate model is
abandoned altogether, and is replaced by multiple local models managed
with a trust-region-like algorithm.

We define the following algorithmic parameters, utilized in the
algorithm.
\begin{itemize}
\item $\beta_{mr}$: minimum radius of the refinement search.
\item $\beta_{rm}$: (logarithm of the) radius multiplier for initialization.
\item $\kappa_{rs}$: threshold to shrink the refinement search radius.
\item $\kappa_{re}$: threshold to expand the refinement search radius.
\item $\kappa_{rm}$: threshold to accept the new iterate.
\item $T_{rf}$: frequency parameter of the refinement search.
\item $T_{rs}$: maximum number of consecutive refinement iterations.
\item $\epsilon_{\text{grad}}$: minimum norm of the gradient of the linear
  model.
\end{itemize}

The Refinement step works as follows:
\begin{itemize}
\item {\bf Model initialization:} Let $j \leftarrow \arg
  \min_{i=1,\dots,k} f(x^i)$. Sort the points $x^1,\dots,x^k$ by
  increasing distance from $x^j$, and select the first $n+1$ (this
  includes $x^j$ itself).  Let $S$ be the set containing these
  points. Set $\bar{x} \leftarrow x^j$.
\item Let $\hat{x}$ be the point in $S$ with the $\ceil{\frac{n+1}{2}}$
  smallest distance to $\bar{x}$. Compute the initial radius of
  the refinement search $\rho$ as:
  \begin{equation*}
    \rho = \max\{ \|\bar{x} - \hat{x}\|, \beta_{mr} \times
    2^{\beta_{rm}}\}.
  \end{equation*}
\item {\bf Refinement:} repeat a given number of times, or until a
  stopping criterion is met.
  \begin{enumerate}[(i)]
  \item Let $M$ be the matrix obtained using the points $x^i \in S$
    as columns.
  \item If $M$ does not contain $n+1$ affinely independent columns, use
    a $QR$ factorization of $M$ to replace one point in $S$ with a new
    point (obtained by moving from $\bar{x}$ in a direction taken from
    the columns of $Q$ after rescaling, with step length $\rho$) that
    increases the rank of $M$, and go back to (i).
  \item Otherwise, build a linear model $c^{\top} x + b$ of the
    objective function using points $(x^i, f(x^i)), x^i \in S$.
  \item Move from the current iterate $\bar{x}$ in the direction of
    improvement $-c$ with step length:
    \begin{equation*}
      t = \max_{0 \le t \le \rho} \{ t : \bar{x} - t c
      \in [x^{eL}, x^{e,U}]\}.
    \end{equation*}
    Let $\bar{x}' = \bar{x} - t c$ be the new candidate point.
  \item Evaluate $f(\bar{x}')$. Update the refinement search radius
    based on the expected decrease $c^{\top} (\bar{x} - \bar{x}')$ and
    the actual decrease $f(\bar{x}) - f(\bar{x}')$: if
    $\frac{f(\bar{x}) - f(\bar{x}')}{c^{\top} (\bar{x} - \bar{x}')}
    \le \kappa_{rs}$, set $\rho \leftarrow \rho/2$, if
    $\frac{f(\bar{x}) - f(\bar{x}')}{c^{\top} (\bar{x} - \bar{x}')}
    \ge \kappa_{re}$ set $\rho \leftarrow 2\rho$.
  \item If $\frac{f(\bar{x}) - f(\bar{x}')}{c^{\top} (\bar{x} -
    \bar{x}')} \ge \kappa_{rm}$ , set $\bar{x} \leftarrow
    \bar{x}'$. 
  \item Replace the point in $S$ furthest from $\bar{x}$ with the new
    point $\bar{x}'$, if it is closer, and go back to (i).
  \end{enumerate}
\end{itemize}

The Refinement step is triggered after $T_{rf}$ full cycles of the
global search strategy in the Iteration step (i.e., the strategy to
select $f^*_k$ in Gutmann's RBF method, or $\alpha$ in MSRSM), but
only if one of the following two conditions apply: (i) a better
solution was discovered since the last execution of the Refinement
step, or (ii) the last Refinement step was stopped because of its
iteration limit (parameter $T_{rs}$, see below), rather than for lack
of improvement.

When the Refinement step ends, all points at which $f$ has been
evaluated are added to $x^1,\dots,x^k$, and the algorithm goes back to
the Iteration step. The Refinement step ends when one of the following
conditions is verified:
\begin{itemize}
\item after $T_{rs}$ consecutive iterations, unless we are close to
  hitting the limit on the maximum number of objective function
  evaluations, or the CPU time limit (this is defined by a further
  parameter);
\item if the radius $\rho$ of the refinement search drops below
  $\beta_{mr}$;
\item if the norm of the gradient of the linear model drops below
  $\epsilon_{\text{grad}}$.
\end{itemize}

The above scheme is designed with continuous variables in mind, but we
heuristically apply the Refinement step also in the presence of
integer or categorical variables. When the problem has integer or
categorical variables, the Refinement step proceeds as described
above, but every candidate point is rounded to an integer point before
being evaluated with $f$. In particular, every integer variable that
takes on a fractional value in the candidate point, say $\bar{x}_j$,
is rounded down with probability $\ceil{\bar{x}_j} - \bar{x}_j$, and
rounded up with probability $\bar{x}_j - \floor{\bar{x}_j}$; whereas
every unary representation of a categorical variable, say
$(\bar{z}_1,\dots,\bar{z}_{m_h})$ such that $\sum_{j=1}^{m_h}
\bar{z}_j = 1$ in extended space, is rounded to the orthonormal basis
vector $e_i$ with probability $\bar{z}_{i} / \sum_{j=1}^{m_h}
\bar{z}_j$ for all $i=1,\dots,m_h$. The rounding process for integer
and categorical variable is repeated a given number times, and the
point with the best linear model score is chosen as the next
candidate. A similar procedure is applied in step (ii) to the column
of $Q$ that is about to replace one column in $M$: each entry is
projected to the closest feasible vector in extended space, using
$\ell_1$-norm distance.

\subsection{Repairing numerical errors: Restoration step}
Whenever numerical errors are detected in the solution of the linear
system \eqref{eq:phisystem}, we switch to a Restoration step that
works as follows. Given the list of interpolation points
$x^1,\dots,x^k$, for $i=k,k-1,\dots,1$ we heuristically solve the
problem:
\begin{equation*}
  \max_{x \in \Omega_e} \min_{j = 1,\dots,k, j \neq i} \|x - x^j\|,
\end{equation*}
then temporarily replace $x^i$ with the solution to the above problem,
say $\bar{x}$. If the system \eqref{eq:rphisystem} for the points
$x^1,\dots,x^{i-1},\bar{x},x^{i+1},\dots,x^k$ is invertible, we
permanently replace $x^i$ with $\bar{x}$, and the Restoration step is
successful. Otherwise, we reinstate $x^i$ and continue the Restoration
step by decreasing $i$. We remark that several interpolation points may be
added in between successive solutions of \eqref{eq:phisystem}, because
the Refinement step may perform multiple iterations and it does not
recompute the interpolant $s_k$. For this reason, we cannot hope that
removing the last interpolation node is always sufficient to fix
numerical errors.

The rationale for solving a maxmin distance problem when trying to
improve the numerics is that proximity to other interpolation points
necessarily leads to an ill-conditioned linear system: if two points
are very close to each other, the corresponding rows in
\eqref{eq:phisystem} are almost identical. This suggests maximizing
the distance from other interpolation points as the main criterion for
choosing a point. Furthermore, this criterion corresponds to the
``pure exploration'' phase of the MSRSM algorithm, trying to gather
information in unexplored parts of the search space; hence, it
naturally fits into our optimization scheme.

\subsection{Automatic model selection}
\label{s:modelsel}
In order to dynamically choose the surrogate model that appears to be
the most accurate for the problem at hand, we assesses model quality
using a cross validation scheme. This was introduced in
\cite{nannirbfopt}: here we give a brief summary of the main ideas,
and report some additional implementation details introduced
subsequently.

Suppose we have $k$ interpolation points $x^1,\dots,x^k$ with
surrogate model $s_k$. We assume that the points are sorted by
increasing function value: $f(x^1) \le f(x^2) \le \dots \le f(x^k)$;
this is without loss of generality as we can always rearrange the
points. We perform cross validation as follows. For
$j\in\{1,\dots,k\}$, we fit a surrogate model $\tilde{s}_{k,j}$ to
the points $(x^i, f(x^i))$ for $i=1,\dots,k, i\neq j$ and evaluate the
performance of $\tilde{s}_{k,j}$ at $(x^j, f(x^j))$. We use an
order-based measure to evaluate performance of the surrogate
model. For a given scalar $y$, let $\text{order}_{k,j}(y)$ be the
position at which $y$ should be inserted in the ordered list $f(x^1)
\le \dots \le f(x^{j-1}) \le f(x^{j+1}) \le \dots \le f(x^k)$ to keep
it sorted. Since $\text{order}_{k,j}(f(x^j)) = j$, we use the value
$q_{k,j} = |\text{order}_{k,j}(\tilde{s}_{k,j}(x^j)) - j|$ to assess
the predictive power of the model. We then average $q_{k,j}$ with $j$
ranging over some subset of $\{1,\dots,k\}$ to compute a model quality
score. This approach is a variation of leave-one-out cross validation
in which we look at how the surrogate model ranks the left-out point
compared to the other points, rather than evaluating the accuracy of the
prediction in absolute terms. This is motivated by the observation
that for the purpose of optimization, a surrogate model that ranks all
points correctly is arguably more useful than a surrogate model that
attains small absolute errors, but is not able to predict how points
compare to each other \cite{audet2018order}.

We perform model selection at the beginning of every cycle of the
search strategy to select $f^\ast_k$ or $\alpha$ (depending on the
choice of algorithm: Gutmann's RBF or MSRSM). Our aim is to select the
RBF model with the best predictive power. We choose two different
models: one for local search, one for global search, corresponding to
different Iteration steps of the algorithm. We do this by computing
the average value $\bar{q}_{10\%}$ of $q_{k,j}$ for
$j=1,\dots,\floor{0.1k}$, and the average value $\bar{q}_{70\%}$ of
$q_{k,j}$ for $j=1,\dots,\floor{0.7k}$. 

The RBF model with the lowest value of $\bar{q}_{10\%}$ is employed in
the subsequent optimization cycle for the {\it Local search} step and
the {\it Global search} step with $h= \kappa-1$, while the RBF model
with lowest value of $\bar{q}_{70\%}$ is employed for all the
remaining steps. We consider all RBFs listed in Table
\ref{tab:rbf_degree}. This implies that the type of RBF 
dynamically changes during the course of the optimization.

In \cite{nannirbfopt}, we show that the values $\bar{q}_{10\%},
\bar{q}_{70\%}$ can be computed in time $O(m^3)$, where $m$ is the
number of rows of \eqref{eq:phisystem} (i.e., $m = k + n + 1$ for
cubic and thin plate spline RBF, $m = k + 1$ for linear and
multiquadric, $m = k$ for Gaussian). This is achieved by reusing the
same LU factorization of the system \eqref{eq:phisystem} for each
iteration of the cross validation routine. Details of this approach
are given in \cite{nannirbfopt}.

When automatic model selection is enabled, we build the surrogate
model using thin plate splines until there are enough points to start
the automatic model selection procedure. Furthermore, after $T_{mcv}$
executions of the automatic model selection procedure, where $T_{mcv}$
is a parameter, we trust the results obtained up to that point and use
the type of RBFs that gave the smallest error the largest number of
times. Results for the local search model and global search model are
kept separate. In other words, the quantities $\bar{q}_{10\%}$ and
$\bar{q}_{70\%}$ are computed at most $T_{mcv}$ times; after that, we
always use the RBF type that gave the smallest value of
$\bar{q}_{10\%}$ the largest number of times out of $T_{mcv}$ for
local search, and similarly with $\bar{q}_{70\%}$ for global
search. This can lead to large time savings on problems with several
thousand interpolation points, as leave-one-out can become expensive
if it has to perform thousands of iterations with a large system
\eqref{eq:phisystem}.

\subsection{Parallel optimizer}
Our implementation supports asynchronous parallel evaluation of the
objective function $f$, which is assumed to be the most time-consuming
part of the optimization process. The parallel optimization algorithm
is nondeterministic due to its asynchronous nature. This algorithm was
first introduced in \cite{nanninnrbfopt}; here we give a brief
overview, as well as several implementation details that were not
present in the version of \cite{nanninnrbfopt}.

The parallel optimizer works by creating a set of worker threads,
coordinated by a master. The worker threads perform tasks of two
types: Type 1 is the evaluation of the objective function at a given
point (which is assumed to be time-consuming), Type 2 is the
computation of a point at which the objective function should be
evaluated (which usually takes only a fraction of a second, but may
take longer especially when some subproblems are solved with
Bonmin). We always dedicate one worker to perform tasks of Type 1 or
of Type 2 related to the Refinement step; however, there is a global
limit to the fraction of Refinement steps that can be performed as
compared to the total number of iterations. The remaining workers are
utilized for the Iteration step. As long as there are available
processors, the master removes a task from the queue of active tasks,
and assigns it to a worker. Tasks of Type 1 have priority over Type 2,
due to their longer execution times. Within tasks of the same type, a
first come, first served policy is used.

Recall that to compute the surrogate model $s_k$ we need pairwise
distinct points (possibly affinely independent, depending on the
degree of the polynomial tail). To ensure that the same point is not
evaluated twice in parallel, whenever a task of Type 2 is completed,
by determining a point $x^{k+1}$ at which $f$ should be evaluated
next, we add a temporary interpolation point at $x^{k+1}$, with value
$\max\{\min_{i=1,\dots,k} \{f(x^i)\}, s_k(x^{k+1})\}$. This point is
converted to a regular interpolation point when the corresponding
objective function evaluation (task of Type 1) is complete, and it is
assigned its true function value $f(x^{k+1})$.

For the Refinement step, a new sample set for the linear model is
computed from scratch every time that a point with better objective
function value is discovered outside the Refinement step; this is
different from the serial optimization algorithm, where the Refinement
step is executed in consecutive iterations and no such event can
occur. Another major difference in the parallel optimizer is that we
do not perform the Restoration step: when \eqref{eq:phisystem} cannot
be solved and the queue of active tasks is empty, we restart the
algorithm. (If the queue is not empty, the algorithm keeps processing
tasks that have finished until \eqref{eq:phisystem} can be solved, or
the queue is empty.) The choice to restart, rather than attempt a
Restoration step, has several motivations. The main reason is that
removing an interpolation node requires synchronizing all threads to
ensure exclusive access to the relevant data structures; because
function evaluations can be very time-consuming, this may leave
several threads idle for a long time. Another reason is that, since
multiple point evaluations are performed in parallel, it is possible
that other points in the queue lead to numerical instability: the
Restoration step would have to eliminate all of them, potentially
provoking a prolonged period of inefficient CPU use. Finally, the
Restoration step is not guaranteed to work unless we allow removing
multiple points from the set $\{x^1,\dots,x^k\}$\footnote{This is also the case
  for serial (i.e., non parallel) optimization, but in the serial case
  removing the most recent interpolation point yields an invertible
  system \eqref{eq:phisystem} in all cases except when multiple points
  are added in consecutive Refinement steps.}; however, choosing a subset of
points to remove is a difficult combinatorial problem, hence we opt
for a simpler approach.

\section{Computational experiments}
\label{s:exp}
RBFOpt is implemented in Python and available on GitHub; it can be
automatically installed from PyPI using {\tt pip}. In this section we
evaluate the computational performance of the solver, with a focus on
testing the features of the optimization algorithm described in
Sect.~\ref{s:algorithm}. All experiments are run on identical virtual
machines with (virtual) Intel Xeon E5-2683 v4 CPUs, clocked at 2.10GHz
and running Linux; these machines are instantiated on an IBM cloud.
We use Ipopt \cite{ipopt} and Bonmin \cite{bonmin} to solve all
auxiliary subproblems that require a mathematical programming solver
(Bonmin is used only if the subproblem has integer variables). Note
that these subproblems are not necessarily solved to optimality, as
they are generally nonconvex (e.g., minimizing the surrogate model
$s_k$ or solving \eqref{eq:gutmannnext}); we put a time limit of 20
seconds on each execution of the solvers, and Bonmin is configured
with the ``B-BB'' algorithm.

\subsection{Test instances}
We test the algorithm on a set of 54 instances, with the following
characteristics:
\begin{itemize}
\item 22 instances have continuous variables only, but no integer or
  categorical variables;
\item 20 instances have integer variables only, or continuous and integer variables, but no categorical variables;
\item 12 instances have any combination of variable types and have at
  least one categorical variable.
\end{itemize}
In Table \ref{tab:instances} we give details on the number of
variables and the source of each problem. All these problems are
highly nonconvex, and their dimension is relatively small. The
instances with categorical variables are obtained by modifying other
problem instances, easily identified by their names. The categorical
variables determine one or both of the following: (1) they modify some
of the problem's data, i.e., vectors of coefficients that appear in
the cost function; (2) they modify some of the functions involved in
the expression for the objective function, although they do not modify
their arguments (e.g., the objective function contains an expression
$g(x_1 + 2 x_2)$, and one of the categorical variables determines what
function $g$ is used among a finite set). All categorical variables in
our set of test problems have at least three possible values, since,
as already mentioned, categorical variables with only two possible
values are modeled as binary variables.

\begin{table}[tb]
\caption{Details of the instances used for the tests. Legend for the sources: S1 is Dixon-Szeg\"{o} \cite{dixon2}, S2 is the original GLOBALLIB, S3 is the MINLPLib 2 \cite{minlplib2}, S4 is \cite{schaffer1986some}, S5 is Neumaier's website \cite{neumaier}, S6 denotes Schoen's smooth functions \cite{tessht}. The functions with no indicated source are discussed in the main text.}

\label{tab:instances}       
\centering
\scriptsize
\begin{tabular}{|l|l|l|l|l||l|l|l|l|l|}
\hline
Instance  & \multicolumn{3}{c|}{\# variables} & Source & Instance  & \multicolumn{3}{c|}{\# variables} & Source  \\
\cline{2-4} \cline{7-9}
& Cont.\ & Int.\ & Cat.\ &  & & Cont.\ & Int.\ & Cat.\ & \\
\hline
branin & 2 & 0 & 0 & S1 &                    nvs06 & 0 & 2 & 0 &  S3 \\                       
camel & 2 & 0 & 0 & S1 &                     nvs07 & 0 & 3 & 0 &  S3 \\                       
ex4\_1\_1 & 1 & 0 & 0 & S2 &                 nvs09 & 0 & 10 & 0 &  S3 \\                      
ex4\_1\_2 & 1 & 0 & 0 & S2 &                 nvs14 & 0 & 5 & 0 &  S3 \\                       
ex8\_1\_1 & 2 & 0 & 0 & S2 &                 nvs15 & 0 & 3 & 0 &  S3 \\                       
ex8\_1\_4 & 2 & 0 & 0 & S2 &                 nvs16 & 0 & 2 & 0 &  S3 \\                       
goldsteinprice & 2 & 0 & 0 &  S1 &           prob03 & 0 & 2 & 0 & S3 \\                       
hartman3 & 3 & 0 & 0 & S1 &                  schoen\_6\_1\_int & 2 & 4 & 0 & S6 \\            
hartman6 & 6 & 0 & 0 &  S1 &                 schoen\_6\_2\_int & 2 & 4 & 0 & S6 \\            
least & 3 & 0 & 0 & S2 &                     schoen\_10\_1\_int & 4 & 6 & 0 & S6 \\           
perm\_6 & 6 & 0 & 0 & S5 &                   schoen\_10\_2\_int & 4 & 6 & 0 & S6 \\           
perm0\_8 & 8 & 0 & 0 & S5  &                 sporttournament06 & 0 & 15 & 0 & S3 \\           
rbrock & 2 & 0 & 0 & S2  &                   st\_miqp1 & 0 & 5 & 0 & S3 \\                    
schaeffer\_f7\_12\_1 & 12 & 0 & 0 & S4 &     st\_miqp3 & 0 & 2 & 0 & S3 \\                    
schaeffer\_f7\_12\_2 & 12 & 0 & 0 & S4  &    st\_test1 & 0 & 5 & 0 & S3 \\                    
schoen\_6\_1 & 6 & 0 & 0 & S6 &              branin\_cat & 2 & 0 & 1 & -- \\                  
schoen\_6\_2  & 6 & 0 & 0 & S6 &             ex8\_1\_1\_cat & 2 & 0 & 2 & --\\                
schoen\_10\_1 & 10 & 0 & 0 & S6 &            hartman3\_cat & 3 & 0 & 1 & --\\                 
schoen\_10\_2 & 10 & 0 & 0 & S6 &            hartman6\_cat & 6 & 0 & 1 & --\\                 
shekel10 & 4 & 0 & 0 &  S1 &                 schoen\_10\_1\_cat & 10 & 0 & 2 & --\\           
shekel5 & 4 & 0 & 0 &  S1 &                  schoen\_10\_2\_cat & 10 & 0 & 2 & --\\           
shekel7 & 4 & 0 & 0 &  S1 &                  gear4\_cat & 1 & 4 & 1 & --\\                    
gear & 0 & 4 & 0 & S3 &                      nvs07\_cat & 0 & 3 & 1 & --\\                    
gear4 & 1 & 4 & 0 & S3 &                     nvs09\_cat & 0 & 10 & 1 & --\\                   
nvs02 & 0 & 5 & 0 &  S3 &                    st\_miqp1\_cat & 0 & 5 & 1 & --\\                
nvs03 & 0 & 2 & 0 &  S3 &                    schaeffer\_f7\_12\_1\_int\_cat & 9 & 3 & 1 & --\\
nvs04 & 0 & 2 & 0 &  S3 &                    schaeffer\_f7\_12\_2\_int\_cat & 9 & 3 & 1 & --\\
\hline
\end{tabular}    
\end{table}

We also use a randomized procedure to create larger instances starting
from a base instance, multiplying its dimension by a given positive
integer number. We now give a high-level description of this
procedure; full details can be found in the publicly available source
code, as a precise description is tedious and does not add further
insight. Let $n$ be the number of variables of the base instance with
objective function $f$, and $s$ the size multiplier. The objective
function of the enlarged instance is:
\begin{equation*}
  \sum_{i=1}^s c_i f(x_{(i-1)s + 1}, \dots, x_{(i)s}) + c_{s+1} f(\ell_1(\sum_{j \in R_1} a_{1j} x_j), \dots, \ell_n(\sum_{j \in R_n} a_{nj} x_j)),
\end{equation*}
where $R_1,\dots,R_n$ is a partition of the set $\{1,\dots,sn\}$, the
coefficients $c_i, a_{ij}$ are randomly chosen within a specified
range, $c_i > 0$ for all $i=1,\dots,s+1$, $\sum_{i=1}^{s+1} c_i = 1$,
and the $\ell_i$ are affine functions that map their argument to the
original domain of $f$. In other words, the enlarged objective
function is the sum of several copies of $f$ defined on disjoint sets
of variables, with a copy of $f$ that acts on linear combinations of
all the variables.  By construction, the value of the optimum stays
the same as in the base instance. We finally permute all the variables
in the enlarged instance. Notice that we do not change the variable
type; e.g., if the base instance has $3$ continuous variables and $2$
categorical variables, using dimension multiplier $s=2$ yields an
instance with $6$ continuous variables and $4$ categorical
variables. From an empirical evaluation, the enlarged instances are
much more difficult than the base instances; this is likely due to the
final copy of $f$ that acts on linear combinations of variables, thus
creating interactions between decision variables that may not have
been present in the original instance. Our final test set consists of
all instances listed in Table \ref{tab:instances}, plus all instances
obtained with the above procedure with a size multiplier $s=2$. This
yields 108 problem instances, with a number of variables varying from
1 to 30.

\subsection{Comparison of algorithmic variants}
\label{s:variants}

To compare algorithmic variants of RBFOpt, we plot performance and data
profiles \cite{morewild09}, which are defined as follows. Define the
{\em budget} for an algorithm as the maximum number of function
evaluations allowed. Unless specified otherwise, in our experiments
the budget is set to $50(n + 1)$. For a given instance and a set of
algorithms ${\cal A}$, let $f^*$ be the best function value discovered
by any algorithm, and $x_0$ the first point evaluated by each
algorithm, which we impose to be the same. Let $0 < \tau < 1$ be a
tolerance. We say that an algorithm {\em solves} an instance up to
tolerance $\tau$ if it returns a point $\bar{x}$ such that:
\begin{equation}
  \label{eq:convergence}
  f(x_0) - f(\bar{x}) \ge (1 - \tau)(f(x_0) - f^*),
\end{equation}
and the algorithm {\em fails} otherwise. In other words, the algorithm
has to close at least $1-\tau$ of the gap between the initial point
and the best point found by any algorithm.

Let ${\cal P}$ be the set of problem instances in the test set. Let
$t_{p,a}$ be the number of function evaluations required by algorithm
$a$ to solve problem $p$ ($t_{p,a} = \infty$ if algorithm $a$ fails on
problem $p$ according to the convergence criterion
\eqref{eq:convergence}), and $n_p$ the number of variables of problem
$p$. The {\em data profile} for an algorithm $a$ is the fraction of
problems that are solved within budget $\alpha(n_p +1)$, defined as:
\begin{equation*}
  d_a(\alpha) := \frac{1}{|{\cal P}|} \left| \left\{ p \in {\cal
    P} : \frac{t_{p,a}}{n_p + 1} \le \alpha \right\} \right|.
\end{equation*}
The performance ratio of algorithm $a$ on problem $p$ is defined as:
\begin{equation*}
  r_{p,a} :=  \frac{t_{p,a}}{\min\{t_{p,a} : a \in {\cal A}\}}.
\end{equation*}
According to this definition, the performance ratio is $1$ for the
best performing algorithm on a problem instance. The performance
profile of algorithm $a$ is defined as the fraction of problems where
the performance ratio is at most $\alpha$, defined as:
\begin{equation*}
  p_a(\alpha) := \frac{1}{|{\cal P}|} \left| \left\{ p \in {\cal
    P} : r_{p,a} \le \alpha \right\} \right|.
\end{equation*}

\begin{figure}[tb]
  \centering
  \subcaptionbox{Performance profile, $\tau = 10^{-2}$ \label{fig:algvariants_1}}{
    \includegraphics[width=0.48\textwidth]{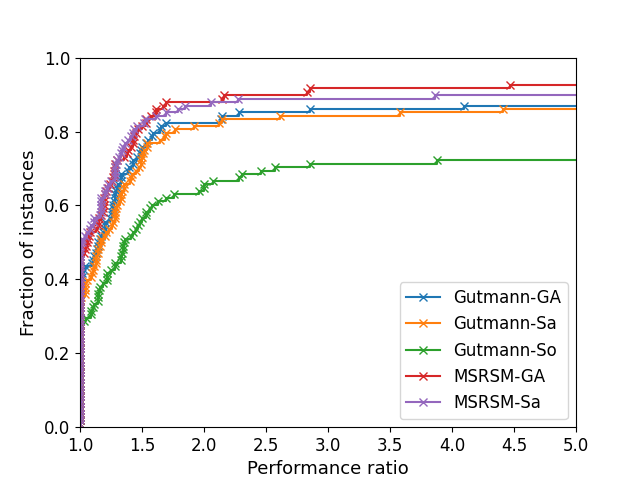}}
  \subcaptionbox{Data profile, $\tau = 10^{-2}$    \label{fig:algvariants_2}}{
    \includegraphics[width=0.48\textwidth]{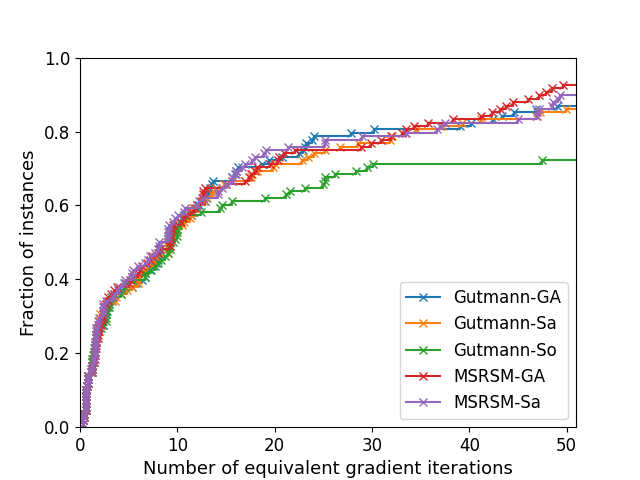}
  }\\
  \subcaptionbox{Performance profile, $\tau = 10^{-4}$    \label{fig:algvariants_3}}{
    \includegraphics[width=0.48\textwidth]{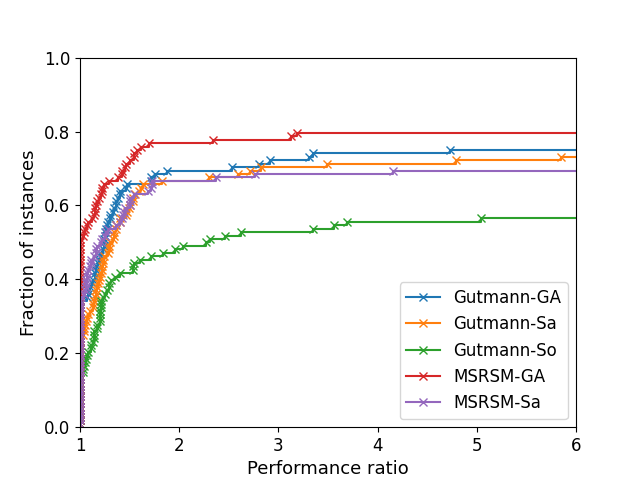}
  }
  \subcaptionbox{Data profile, $\tau = 10^{-4}$    \label{fig:algvariants_4}
}{
    \includegraphics[width=0.48\textwidth]{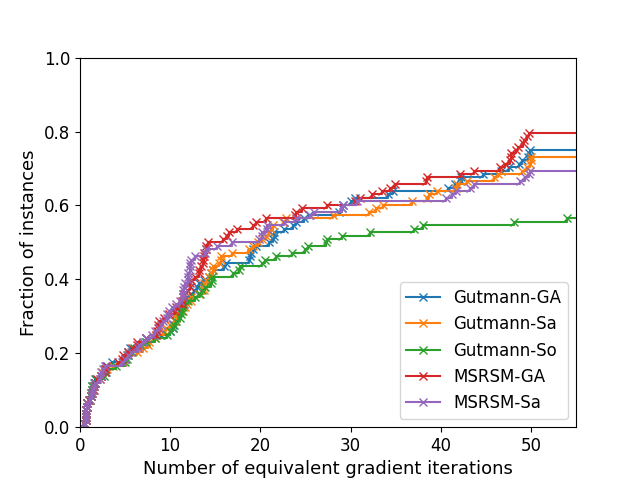}
  }
  
  \caption{Performance profiles (left) and data profiles (right) for different Iteration step procedures. Legend: GA = genetic algorithm, Sa = sampling method, So = mathematical optimization solver.}
  \label{fig:algvariants}
  
\end{figure}

For each of the 108 problem instances, we test 20 different random
seeds. All tested variants are given the same sequence of random
seeds. We remark that if two variants of the algorithm use the same
number of points in the initialization phase and have the same random
seed, then they will generate exactly the same initial sample set. For
every instance, we aggregate the 20 different random seeds by taking
the median objective function value at every iteration; the
performance and data profiles are constructed using the aggregate
data. In this section we use the serial version of the optimization
algorithm.

In our first set of experiments we compare the two methodologies for
the Iteration step discussed in Sect.~\ref{s:iteration} (i.e.,
Gutmann's method and MSRSM), combined with the three approaches to
solve the resulting subproblems discussed in
Sect.~\ref{s:subproblems}: the genetic algorithm, the sampling method,
and the mathematical optimization solver. We remark that the
mathematical optimization solvers are relatively slow, taking up to 20
seconds per solve on the more difficult problems, whereas the genetic
algorithm and the sampling method only require a fraction of a second
due to their heuristic nature. For this set of experiments only, we
parametrize the genetic algorithm and the sampling method in a
search-intensive fashion, increasing the number of sampled points and
the number of iterations of the genetic algorithm compared to their
default values (the genetic algorithm uses a base population size of
$5000+n/5$ points and performs $40$ iterations, compared to a default
of $400 + n/5$ and $20$ iterations, while the sampling algorithm
samples $3000n$ points, compared to a default of $1000n$).  All the
other parameters for the algorithm are left to their default
values. We also remark that the implementation of the MSRSM method
with mathematical optimization solvers is not competitive with the
other variants, because the solution of \eqref{eq:msrsmnext} with a
solver for convex problems is essentially hopeless: the expression of
the maxmin distance is highly nonconvex and solvers have a very high
chance of getting trapped in poor local minima. Hence, we only report
results for five algorithmic variants: all combinations of Gutmann's
method and MSRSM with the genetic algorithm, the sampling approach,
and the mathematical optimization solver approach, minus the
combination MSRSM + mathematical optimization solver. Results are
plotted in Fig.~\ref{fig:algvariants}. The plots quite convincingly
show that the genetic algorithm is overall the best choice, with both
Gutmann's method and MSRSM. The sampling method has a similar
performance, while using the mathematical optimization solver is
considerably worse. We attribute this to the fact that the subproblems
involved in the Iteration step are hard nonconvex problems, and the
solvers are likely to struggle. (We remark that the minimization of
the surrogate model is always performed using Ipopt or Bonmin,
regardless of the methodology used to solve subproblems in the
Iteration step.)  The difference between Gutmann's method and MSRSM is
small, but MSRSM emerges as the winner by a small margin. In the
following, we use MSRSM with the genetic algorithm as the default
settings. An important conclusion of our numerical study is the fact
that neither method dominates the other: while MSRSM seems slightly
better and enjoys the benefit of being conceptually simpler, our plots
indicate that Gutmann's method is competitive.

\begin{figure}[tb]
  \centering
  \subcaptionbox{Performance profile, $\tau = 10^{-2}$     \label{fig:init_1}}{
    \includegraphics[width=0.48\textwidth]{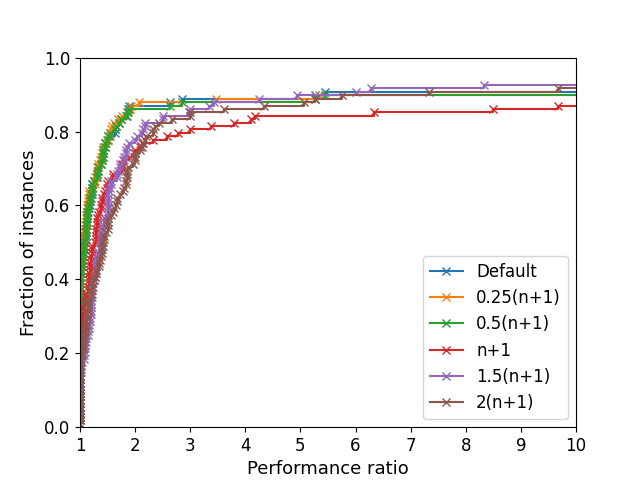}
  }
  \subcaptionbox{Data profile, $\tau = 10^{-2}$    \label{fig:init_2}}{
    \includegraphics[width=0.48\textwidth]{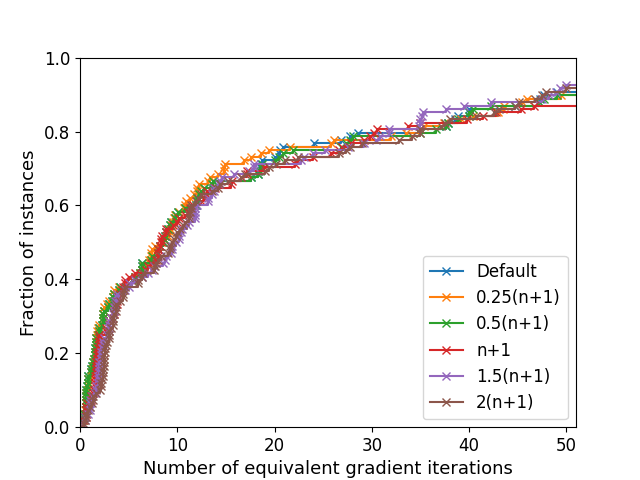}
  }\\
  \subcaptionbox{Performance profile, $\tau = 10^{-4}$    \label{fig:init_3}}{
    \includegraphics[width=0.48\textwidth]{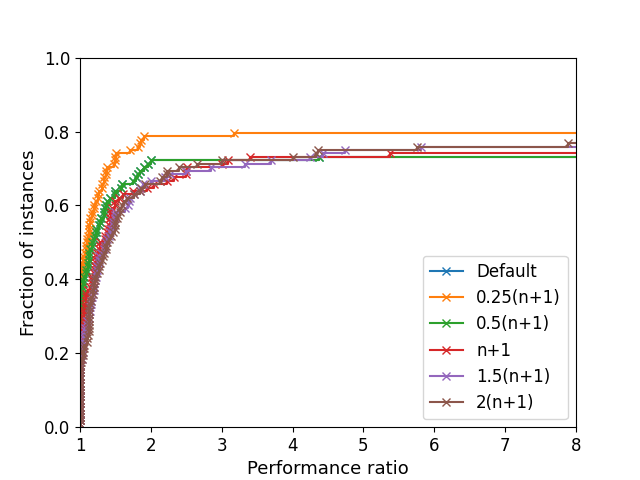}
  }
  \subcaptionbox{Data profile, $\tau = 10^{-4}$    \label{fig:init_4}}{
    \includegraphics[width=0.48\textwidth]{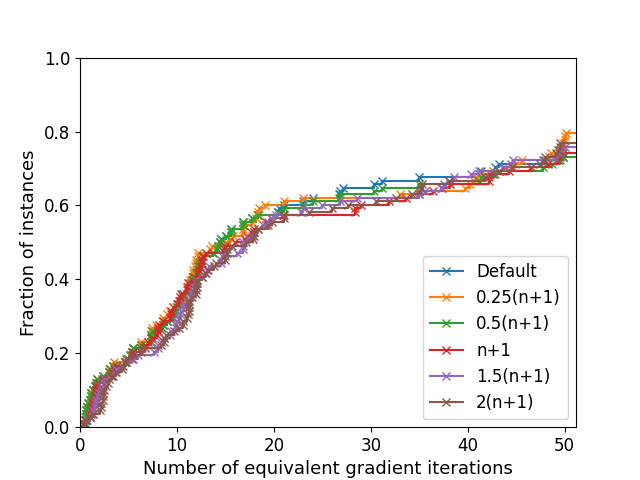}
  }
  
  \caption{Performance profiles (left) and data profiles (right) using a different number of sample points in the initialization procedure.}
  \label{fig:init}
  
\end{figure}

In the second set of experiments we analyze the impact of the number
of sample points for the Initialization step of the algorithm. As
discussed in Sect.~\ref{s:lssol}, we allow building a surrogate model
with less than $n+1$ points, in which case the system
\eqref{eq:phisystem} may have multiple solutions. In
Fig.~\ref{fig:init} we report results when using $0.25(n+1), 0.5(n+1),
n+1, 1.5(n+1), 2(n+1)$ sample points to initialize $s_k$, as well as
the number of points defined in \eqref{eq:initformula}, labeled
``Default'' in the plots. We can see that $0.25(n+1), 0.5(n+1)$, and
``Default'' have the best performance with $\tau = 10^{-2}$, and there
is no winner among these three. For $\tau = 10^{-4}$, the curve for
``Default'' is not visible in Fig.~\ref{fig:init_3} because it is
hidden behind the curve for $0.5(n+1)$: this is expected, since by
equation \eqref{eq:initformula}, ``Default'' uses $0.5(n+1)$ on most
problem instances. Choosing $0.25(n+1)$ emerges as the winner in these
tests, followed by $0.5(n+1)$ and ``Default'', which are
indistinguishable on this set of test problems. The motivation for the
``Default'' setting, which seems slightly inferior to $0.25(n+1)$ in
these tests, is robustness: using a very small number of sample points
can increase the variance of the algorithm, hence we prefer the safer
setting.

\begin{figure}[htb]
  \centering
  \subcaptionbox{Performance profile, $\tau = 10^{-2}$    \label{fig:refinement_1}}{
    \includegraphics[width=0.48\textwidth]{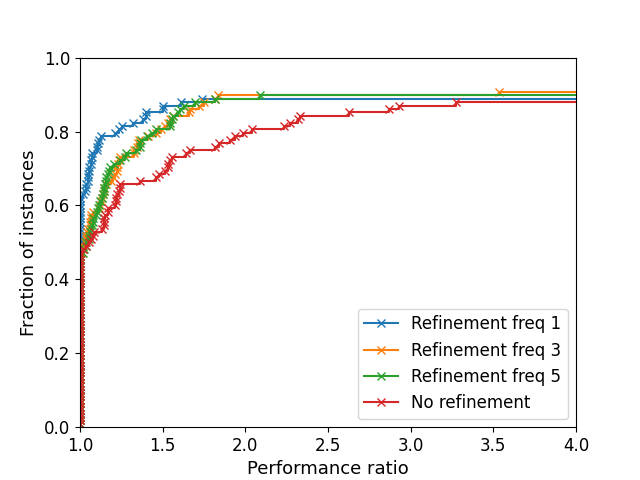}
  }
  \subcaptionbox{Data profile, $\tau = 10^{-2}$    \label{fig:refinement_2}}{
    \includegraphics[width=0.48\textwidth]{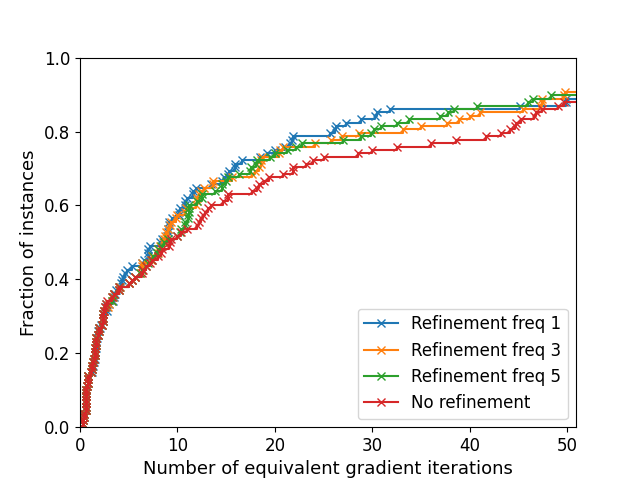}
  }\\
  \subcaptionbox{Performance profile, $\tau = 10^{-4}$    \label{fig:refinement_3}}{
    \includegraphics[width=0.48\textwidth]{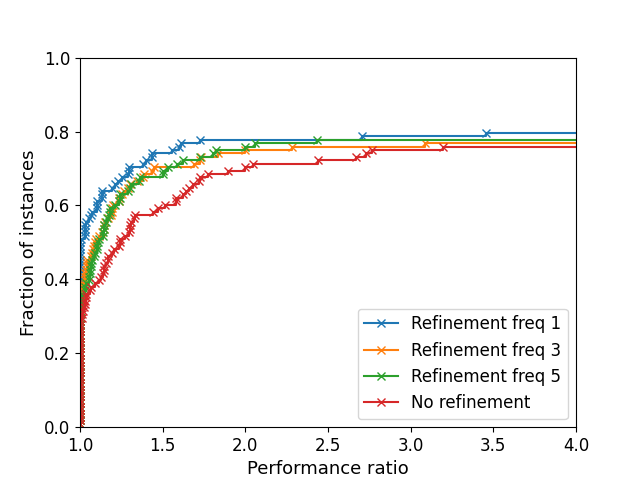}
  }
  \subcaptionbox{Data profile, $\tau = 10^{-4}$    \label{fig:refinement_4}}{
    \includegraphics[width=0.48\textwidth]{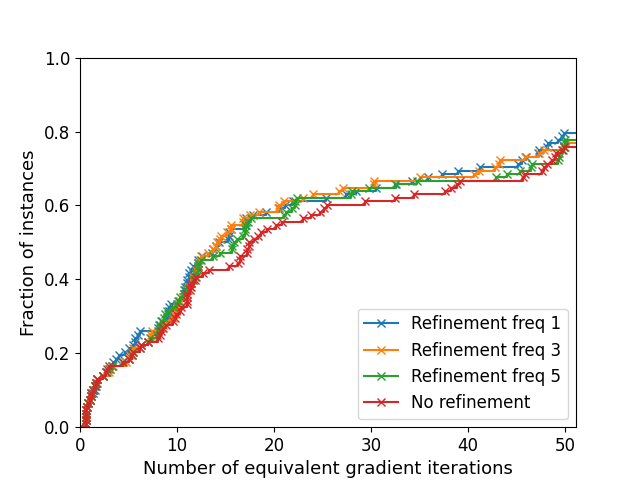}
  }
  
  \caption{Performance profiles (left) and data profiles (right) with and without the Refinement step. The ``Frequency'' of the Refinement step indicates after how many Iteration steps it is performed.}
  \label{fig:refinement}
  
\end{figure}

In the last set of experiments for this section, we look at the impact
of the Refinement step. Plots are reported in
Fig.~\ref{fig:refinement}. We compare RBFOpt without the Refinement
step, with three versions of the algorithm that employ the Refinement
step at different frequencies. Here the results are very clear: the
Refinement step significantly improves the performance of the
algorithm on this set of instances. The plots suggest that running the
Refinement step as frequently as possible is a good idea. A
head-to-head comparison\footnote{It is known that performance profiles
  depend on the entire set of algorithms evaluated; therefore,
  pairwise comparisons can sometimes yield useful information.}
between the algorithm with Refinement step frequencies of 1 and 3
reveals that the difference is quite small and not as one-sided as it
would appear from Fig.~\ref{fig:refinement}, see Fig.~\ref{fig:refhh}
(results for $\tau=10^{-3}$ are essentially identical to those for
$\tau=10^{-4}$). We set the Refinement frequency to 3 as the default
value, mostly based on empirical evaluation on applications outside
the benchmark set reported here.

\begin{figure}[tbp]
  \centering
  \subcaptionbox{Data profile, $\tau = 10^{-2}$    \label{fig:refhh_1}}{
    \includegraphics[width=0.48\textwidth]{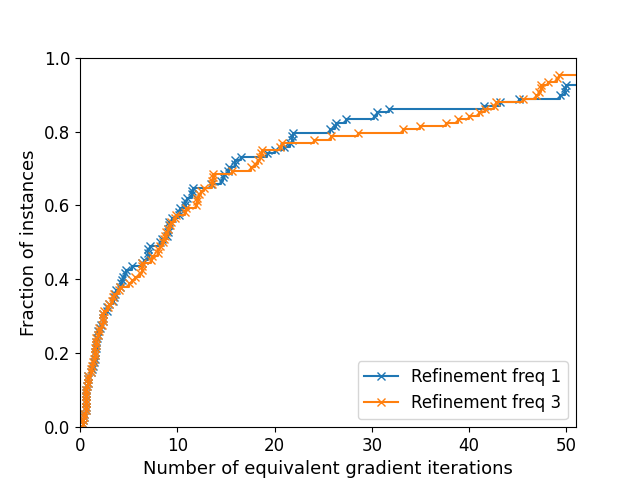}
  }
  \subcaptionbox{Data profile, $\tau = 10^{-4}$    \label{fig:refhh_2}}{
    \includegraphics[width=0.48\textwidth]{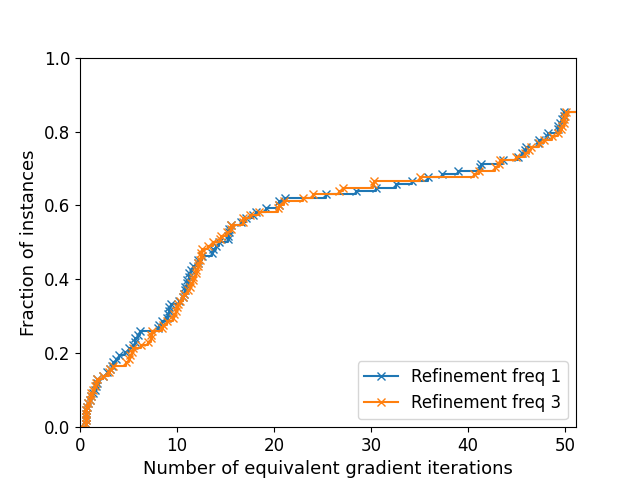}
  }
  
  \caption{Data profiles with different frequencies of the Refinement step.}
  \label{fig:refhh}
  
\end{figure}

\subsection{Categorical variables: original versus extended space}
\label{s:expcat}
In this section we analyze the performance of the optimization
algorithm with the two different representations for categorical
variables discussed in Sect.~\ref{s:extspace}. We use the same
approach as in the previous section; all algorithm parameters are set
to their default values indicated in the previous section. For this
set of experiments we use the problem instances with categorical
variables only, see Table \ref{tab:instances}, as well as their
enlarged version with dimension multiplier $s=2$. To reduce variance,
we use the same points for the Initialization step, regardless of the
choice of extended or original space. This is accomplished as follows:
for every instance and every random seed, we generate the initial
samples in extended space (the number of samples is chosen according
to \eqref{eq:initformula}); we then map these points to their
equivalent in the original space, and use them to initialize the
optimization in original space. As a consequence, the optimization in
original space uses more initial samples than it normally
would. Results are reported in Fig.~\ref{fig:cat}.

\begin{figure}[tbp]
  \centering
  \subcaptionbox{Performance profile, $\tau = 10^{-2}$    \label{fig:cat_1}}{
    \includegraphics[width=0.48\textwidth]{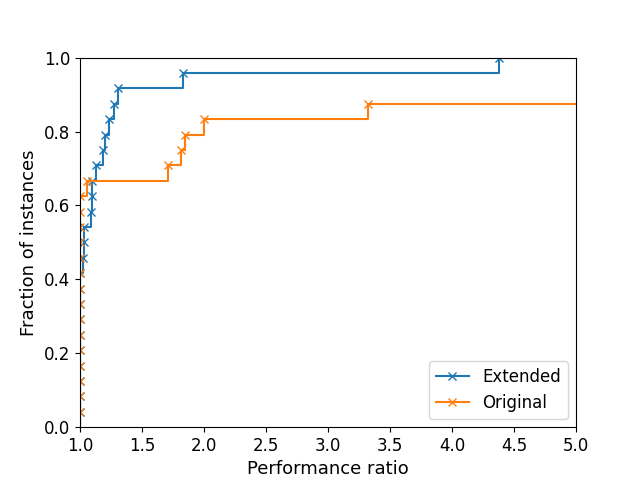}
  }
  \subcaptionbox{Data profile, $\tau = 10^{-2}$    \label{fig:cat_2}}{
    \includegraphics[width=0.48\textwidth]{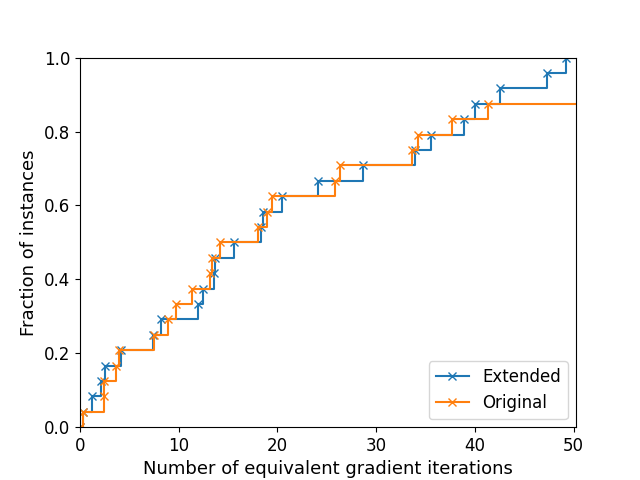}
  }\\
  \subcaptionbox{Performance profile, $\tau = 10^{-4}$    \label{fig:cat_3}}{
    \includegraphics[width=0.48\textwidth]{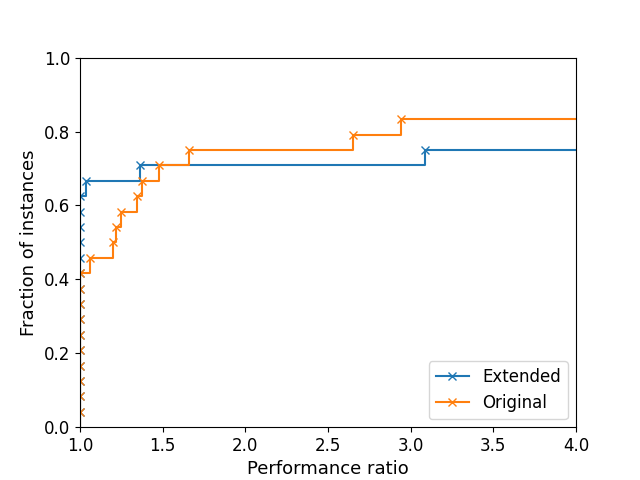}
  }
  \subcaptionbox{Data profile, $\tau = 10^{-4}$    \label{fig:cat_4}}{
    \includegraphics[width=0.48\textwidth]{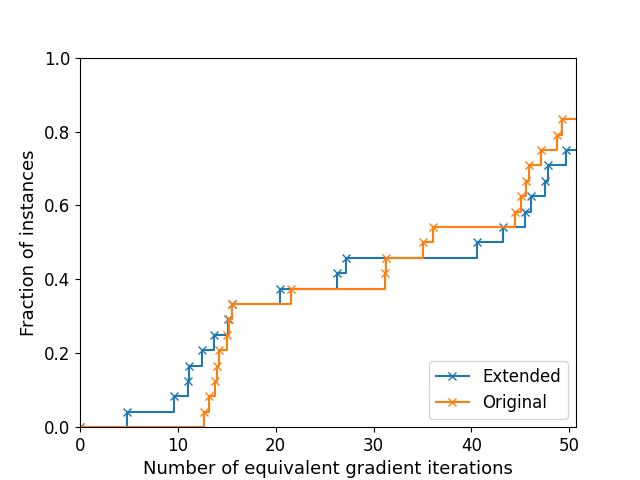}
  }
  
  \caption{Performance profiles (left) and data profiles (right) for optimization in extended and original space.}
  \label{fig:cat}
  
\end{figure}

The plots for $\tau = 10^{-2}$ indicate a clear superiority for
optimization in extended space; with $\tau = 10^{-4}$ the difference
is not so clear, with the extended space formulation showing better
performance on some instances that can be solved quickly, but the
original space formulation converges on more instances in the long
run. A possible explanation for this behavior is the fact that the
optimization algorithm is quite robust to innaccuracies in the
surrogate model (because it favors points with large distance from
those already evaluated, and it performs local search around the best
known point), hence no reasonable formulation for the categorical
variables will perform too poorly. We also remark that with $\tau =
10^{-3}$ (not reported here), the plots are similar to $\tau =
10^{-4}$, but the difference is less pronounced. Overall, the extended
space formulation is to be preferred: it manages to close $99\%$ of
the gap on all test problems (as indicated by the plots for $\tau =
10^{-2}$), and for almost all problems, it does so very quickly; in
the long run the original space formulation solves a few more
instances to high precision, but this does not seem enough to offset
the advantage of the extended space formulation in the initial
iterations.

To compare the original and extended space formulations from a
different angle, we also set up an experiment to assess the usefulness
of the surrogate models constructed in these two spaces. This is
accomplished as follows. For every function with categorical variables
in Table \ref{tab:instances}, we generate an initial sample of $k$
points, with $k \in \{(n+1), 5(n+1), 10(n+1), 50(n+1)\}$. These points
are generated as a latin hypercube design maximizing the minimum
distance between points, and they are generated in extended space. We
then construct a surrogate model interpolating at these points,
generate 20000 additional random points in the domain of the function,
and rank these 20000 points using the surrogate model. More precisely,
assume w.l.o.g.\ (up to reordering) that the interpolation points are
sorted by increasing function value, i.e., $f(x^1) \le f(x^2) \le
\dots \le f(x^k)$; for every point $x$ we infer its position in the
sorted list $f(x^1), \dots, f(x^k)$ by using its surrogate model value
$s_k(x)$. We then compare this number with the true position of $f(x)$
in the sorted list, and record the absolute value of the difference
between the two numbers. This is a measure of how well the surrogate
model is able to rank unseen points as compared to the known
interpolation points. We record the average and standard deviation of
the difference over the 20000 randomly generated points. The same
procedure is repeated in the original space, using exactly the same
points mapped from the extended space. The results are reported in
Table \ref{tab:cat_accuracy}.

\begin{table}[tbp]
  \centering
  {\small
    \setlength{\tabcolsep}{2pt}
  \begin{tabular}{|l|r|r|r|r|r|r|r|r|}
    \hline
    & \multicolumn{2}{c|}{$(n+1)$ points} & \multicolumn{2}{c|}{$5(n+1)$ points} & \multicolumn{2}{c|}{$10(n+1)$ points} & \multicolumn{2}{c|}{$50(n+1)$ points} \\
    \cline{2-9}
    RBF type & Original & Extended & Original & Extended  & Original & Extended  & Original & Extended \\
    \hline
    Cubic & 4.1 (1.8) & 4.2 (1.9) & 15.7 (8.0) & 14.3 (8.2) & 28.8 (16.2) & 25.4 (15.9) & 120.5 (83.5) & 102.9 (80.1) \\
    Gaussian & 4.8 (2.3) & 4.5 (2.4) & 20.9 (11.9) & 20.2 (12.3) & 39.5 (25.2) & 38.1 (25.9) & 193.1 (131.8) & 186.7 (133.6) \\
    Linear & 3.5 (1.4) & 3.3 (1.3) & 15.2 (7.9) & 14.7 (7.7) & 27.1 (15.1) & 25.5 (14.5) & 119.1 (81.5) & 107.6 (78.2) \\
    Multiquad. & 3.4 (1.4) & 3.3 (1.3) & 15.4 (8.0) & 14.5 (7.7) & 27.0 (15.2) & 25.4 (14.6) & 118.1 (82.4) & 107.6 (79.0) \\
    Thin pl.\ sp. & 4.1 (1.8) & 4.2 (1.9) & 15.4 (8.0) & 13.9 (8.0) & 28.1 (16.2) & 25.0 (15.9) & 117.6 (82.1) & 101.7 (77.9) \\
    \hline
  \end{tabular}
  }
  \caption{Average (standard deviation) of the absolute difference between the true rank and the inferred rank of random points, in the original and the extended space.}
  \label{tab:cat_accuracy}
\end{table}

Even though the standard deviations are fairly high, the averages
indicate that the extended space formulation is able to better predict
the rank of unseen points. Indeed, the average rank errors are smaller
for the extended space model in 18 out of the 20 cases reported in
Table~\ref{tab:cat_accuracy}, and the only two cases in which the
extended space has higher average error are recorded when the number
of interpolation points is small ($n+1$), so that the differences
between extended and original space are small in the absolute
sense. For a larger number of interpolation points, the surrogate
model in extended space is consistently better. This, together with
the previous set of experiments, and the theoretical justification
given in Sect.~\ref{s:extspace}, gives a strong indication of the
benefits of our choice.

\subsection{Comparison with existing open-source derivative-free solvers}
\label{s:litcomp}
We compare the performance of RBFOpt, parametrized according to the
results discussed in previous sections, with several derivative-free
solvers that support categorical variables, namely:
\begin{itemize}
\item Nevergrad \cite{nevergrad} version 0.4.0, a collection of
  evolutionary algorithms for hyperparameter optimization; here, we
  test three algorithms that are recommended by \cite{nevergrad} for
  their versatility and generally good performance: OnePlusOne, PSO
  (particle swarm optimization), and TwoPointsDE.
\item NOMAD \cite{nomad} version 3.9.1, an implementation of the mesh adaptive
  direct search algorithm \cite{audetdennis04}.
\item Optuna \cite{optuna2019} version 2.3.0, a hyperparameter
  optimization algorithm that uses a tree-structured Parzen estimator
  to deal with categorical variables.
\item Scikit-Optimize \cite{scikit-optimize} version 0.8.1, a Bayesian
  optimization algorithm using Gaussian processes (using the function
  {\tt gp\_minimize}).
\item SMAC (sequential model-based algorithm configuration)
  \cite{hutter11smac} version 0.13.1, another Bayesian optimization
  algorithm using Gaussian processes\footnote{We used the {\tt
      SMAC4BO} interface; at the time of writing this paper, the other
    interfaces did not work in the available beta version of the
    software package.}.
\end{itemize}
This selection covers the most popular methodologies for
derivative-free optimization with categorical variables. We remark
that in the above list, only NOMAD is developed for derivative-free
optimization in the traditional sense, whereas the other software
target hyperparameter optimization problems, which generally have
added complications (e.g., the objective function evaluation are
nondeterministic, and there may be complicated constraints involving
the categorical variables); nonetheless, all these algorithms can be
applied to black-box optimization problems with categorical variables
(a comparison on a hyperparameter optimization problem is given in
Sect.~\ref{s:hpo}). All algorithms except NOMAD provide a Python
library (NOMAD's Python library does not support categorical
variables).

\begin{figure}[htb]
  \centering
  \subcaptionbox{Performance profile, $\tau = 10^{-2}$    \label{fig:ng_1}}{
    \includegraphics[width=0.48\textwidth]{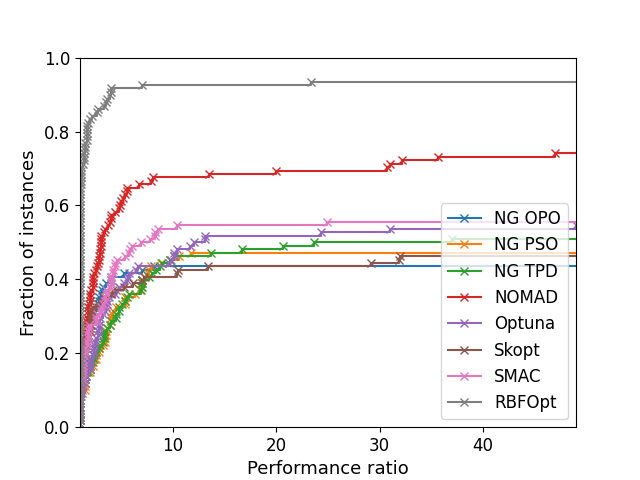}
  }
  \subcaptionbox{Data profile, $\tau = 10^{-2}$    \label{fig:ng_2}}{
    \includegraphics[width=0.48\textwidth]{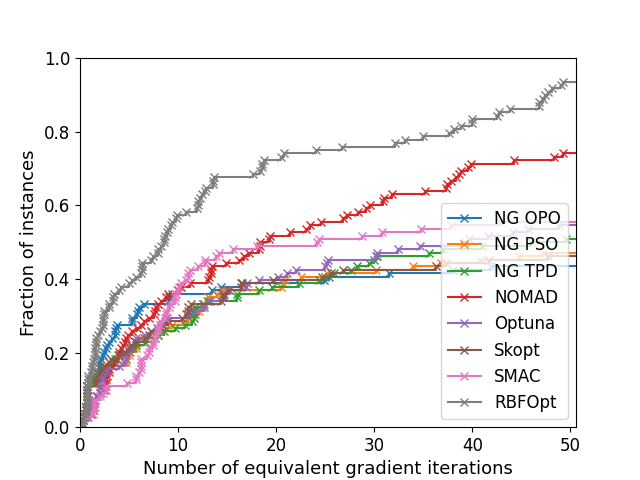}
  }\\
  \subcaptionbox{Performance profile, $\tau = 10^{-4}$    \label{fig:ng_3}}{
    \includegraphics[width=0.48\textwidth]{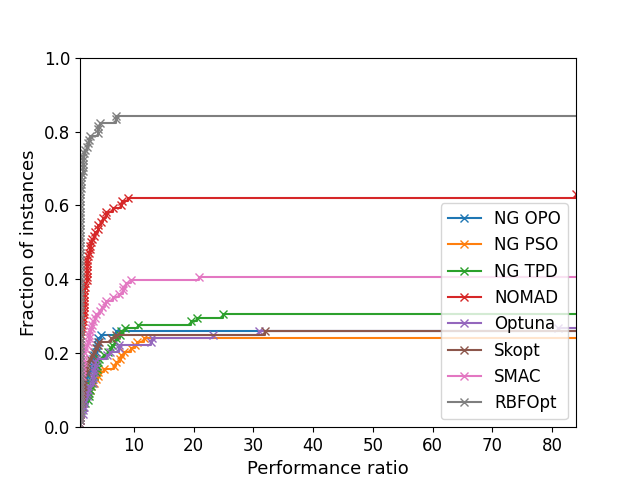}
  }
  \subcaptionbox{Data profile, $\tau = 10^{-4}$    \label{fig:ng_4}}{
    \includegraphics[width=0.48\textwidth]{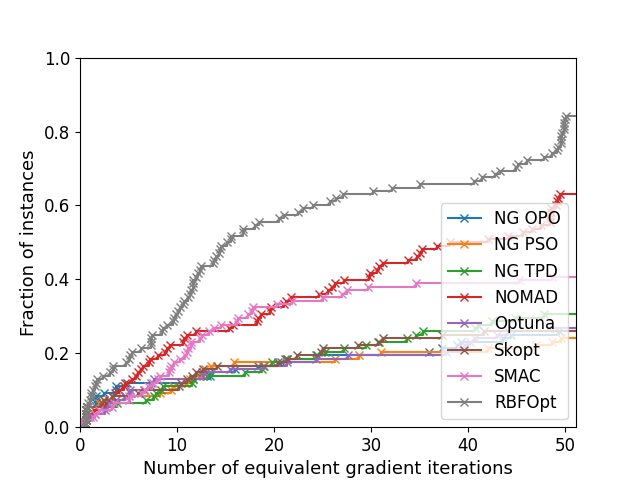}
  }
  
  \caption{Performance profiles (left) and data profiles (right)
    comparing: Nevergrad OnePlusOne, Nevergrad Particle Swarm
    Optimization, Nevergrad TwoPointsDE, NOMAD, Optuna,
    Scikit-Optimize, SMAC, and RBFOpt.}
  \label{fig:ng}
  
\end{figure}

We report performance and data profiles in Fig.~\ref{fig:ng}. Several
remarks are in order to clarify the experimental setup of this
section. First, we try to initialize each algorithm with the same set
of points to reduce variance: RBFOpt, Nevergrad, Scikit-Optimize are
initialized with the experimental design generated by RBFOpt, NOMAD is
initialized with the best point in the experimental design generated
by RBFOpt, while Optuna and SMAC use their own initialization
procedures. Second, due to resource constraints we impose a limit of 3
hours of CPU time for each problem instance; the only two algorithms
that hit the time limit are Scikit-Optimize and SMAC, most likely due
to the optimization of the expected improvement criterion, whereas all
other algorithms are considerably faster (e.g., RBFOpt takes on
average 5 minutes per instance, Scikit-Optimize almost 2 hours on
average). More precisely, Scikit-Optimize times out on 1053 instances,
whereas SMAC times out on 1720 instances, out of 2160. Since we
observed cases where Scikit-Optimize and SMAC take more than a
day to hit the function evaluation limit, the time limit is necessary
to conclude the experimental evaluation within a reasonable time
frame. We remark that the slow down when the number of variables or
the function evaluation budget are large is a known limitation of
Bayesian optimization methods \cite{eriksson2019scalable}. Despite
hitting the time limit, the performance of Scikit-Optimize and SMAC is
comparable to that of the other hyperparameter optimization
algorithms, because the slow down only occurs after a few hundred
function evaluations (the evaluation limit is often $\ge 1000$), which
is sufficient for them to find a good solution.

Fig.~\ref{fig:ng} shows that on this set of test instances, the
comparison is heavily in favor of RBFOpt. NOMAD is the second best
solver, but RBFOpt performs better by a noticeable margin; SMAC is
fairly close to NOMAD for $\tau = 10^{-4}$, whereas all remaining
algorithms are far from achieving the same performance level as the
top solvers.

\subsection{Parallel optimization}
To assess the performance of the parallel version of the optimization
algorithm, we modify our test functions so that each objective
function evaluation waits for $X$ seconds, where $X$ is a random variable,
before returning a value. Thus, we simulate the effects of a
time-consuming objective function oracle. We test two possible
distributions for $X$, both of which are log-normal: in the first case
$\log X$ is distributed as ${\cal N}(3, 0.5)$, where ${\cal N}(\mu,
\sigma)$ is the normal distribution with mean $\mu$ and standard
deviation $\sigma$; in the second case $\log X$ is distributed as
${\cal N}(4, 0.75)$. In both cases we truncate the distributions at
300 seconds, i.e., each objective function evaluation takes at most
300 seconds. Notice that the expected value of the first distribution
is $\approx 20$ (seconds), the expected value of the second
distribution is $\approx 55$. We denote the first case as the ``Faster
evaluation'' set, the second case as the ``Slower evaluation'' set.
We use the same budget of $50(n+1)$ function evaluations, and run the
algorithm with 1, 2, 4, 8 or 16 CPUs\footnote{Although it may seem
  natural to set a function evaluation budget that depends on the
  number of parallel threads, we use a fix budget for practical
  reasons: with 108 problem instances, 20 random seeds for each
  instance, a budget of $50(n+1)$ evaluations, at approximately 1
  minute per function evaluation these experiments already take more
  than one year of CPU time.}. To assess the speedup achieved by the
parallel algorithm, we report the wall-clock time to converge within
$0.1\%$ of the optimal solution. Notice that the larger the number of
CPUs, the faster the evaluation budget is depleted; since the parallel
algorithm is unlikely to be as efficient as the serial version,
as the number of CPUs increases we expect to converge on a smaller
number of instances. Thus, when reporting the average time to
convergence we only consider instances and random seeds for which all
the variants analyzed in this comparison determine the optimum (up to
the specified tolerance). Additionally, we report the number of
instances of which convergence to the specified tolerance is
attained. The data is reported in Table~\ref{tab:parallelcpu}; we use
shifted geometric means for the wall-clock times, defined as
$\left(\prod_{i=1}^k (t_i+1)\right)^{1/k}-1$ for a set of $k$ values
$t_1,\dots,t_k$. In total, there are 568 combinations of instances and
random seeds for which all variants converge on the ``Faster
evaluation'' set, and 550 on the ``Slower evaluation'' set.

\begin{table}
  \centering
  \begin{tabular}{|l|r|r|r|r|r|r|}
  \hline
  Num & \multicolumn{3}{c|}{Faster evaluation ($\approx 20$ sec)}
  & \multicolumn{3}{c|}{Slower evaluation ($\approx 55$ sec)} \\
  \cline{2-7}
  CPUs & Time & Speedup & \# conv.\ & Time & Speedup & \# conv.\ \\
  \hline
  1  &  797.0 & 1.00 & 1255 & 2405.5 & 1.00 & 1255 \\
  2  &  473.5 & 1.68 & 1274 & 1490.3 & 1.61 & 1260 \\
  4  &  307.2 & 2.59 & 1195 &  958.7 & 2.51 & 1175 \\
  8  &  221.2 & 3.60 & 982  &  708.8 & 3.39 & 975  \\
  16 &  176.8 & 4.50 & 635  &  480.1 & 5.01 & 620  \\
  \hline
  \end{tabular}
  \caption{Shifted geometric mean of the wall-clock time to converge to an optimal solution.}
  \label{tab:parallelcpu}
\end{table}

Table~\ref{tab:parallelcpu} shows that parallel optimization is not as
efficient as serial optimization: the speedup for using $c$ CPUs is
roughly $\sqrt{c}$ in our tests. However, in certain applications this
is still a favorable tradeoff, as multiple CPUs are easy to obtain and
wall-clock time can be important. In particular, on this set of test
instances using up to 4 CPUs increases the speed of the optimization,
with a negligible effect on the number of instances on which the
algorithm converges. For 8 or more CPUs the algorithm converges on
significantly fewer instances as compared to serial optimization
(about 50\% of the instances, with 16 CPUs); however, we emphasize
once again that in these tests we keep the same function evaluation
budget for all variants of the algorithm, therefore we are likely to
run out of budget quickly with 8 or 16 CPUs. In other words, the low
number of instances on which the 8-CPU and 16-CPU version of the
parallel optimization algorithm converges implies that parallel
optimization is less efficient than serial optimization for the same
budget, but the significant speedups indicate that is is more
efficient for the same amount of wall-clock time. Finally, changing
the distribution of the objective function evaluation times seems to
have little effect in these tests: in the ``Slower evaluation''
experiments, despite a much larger mean evaluation time and an
increased variance, the recorded speedup factors are very similar to
the ``Faster evaluation'' experiments.

\subsection{Application to hyperparameter optimization}
\label{s:hpo}

\begin{table}[tbp]
  \centering
  \begin{tabular}{|l|r|r|p{7cm}|}
    \hline
    Name & Type & Domain & Description \\
    \hline
        {\tt n\_estimators} & Int. & $\{10, 20, \dots, 1000\}$ & Number of trees \\
        {\tt criterion} & Cat. & $\{0, 1\}$ & Measure of split quality \\
        {\tt max\_depth} & Int. & $\{5, 6, \dots, 99\} \cup \{\infty\}$ & Maximum depth \\
        {\tt min\_samples\_split} & Int. & $\{2, 3, \dots, 20\}$ & Minimum number of samples required to split an internal node \\
        {\tt min\_samples\_leaf} & Int. & $\{1, 2,\dots, 10\}$ & Minimum number of samples required to be a leaf \\
        {\tt min\_weight\_fraction\_leaf} & Real & $[0, 0.5]$ & Minimum weighted fraction of the sum total of weights required to be at a leaf node \\
        {\tt max\_features} & Cat. & $\{0, 1, 2\}$ & Number of features to consider when looking for the best split \\
        {\tt min\_impurity\_decrease} & Real & $[0, 1]$ & A node will be split if this split induces a decrease of the impurity greater than or equal to this value \\
        {\tt class\_weight} & Cat. & $\{0, 1, 2\}$ & Weights associated with classes \\
        {\tt ccp\_alpha} & Real & $[0, 1]$ & Complexity parameter used for minimal cost-complexity pruning \\
        \hline
  \end{tabular}  
  \caption{Hyperparameters of the random forest classifier. For {\tt max\_features}, the choices are $\sqrt{\text{num\_features}}, \log (\text{num\_features}), \text{num\_features}$; for {\tt class\_weight}, the choices are given by three vectors of weights that attempt to rebalance the proportion of samples in each class in different ways (uniform, proportional to frequency, proportional to the square root of the frequency).}
  \label{tab:rfparameters}
\end{table}

\begin{figure}[tbp]
  \subcaptionbox{RBFOpt, original space    \label{fig:rf_1}}{
    \includegraphics[width=0.48\textwidth]{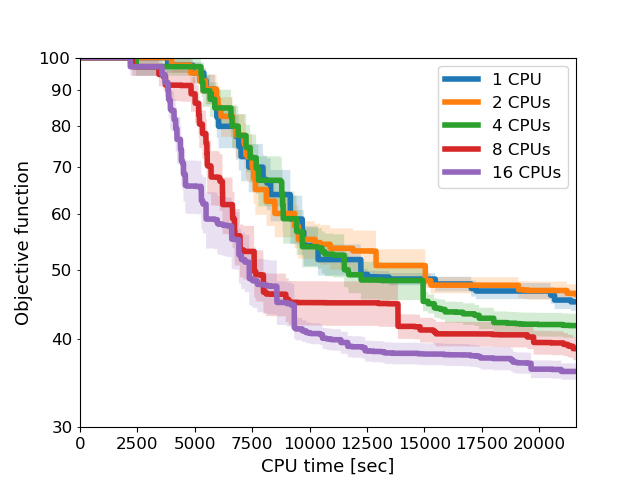}
  }
  \subcaptionbox{Comparison of several algorithms    \label{fig:rf_2}}{
    \includegraphics[width=0.48\textwidth]{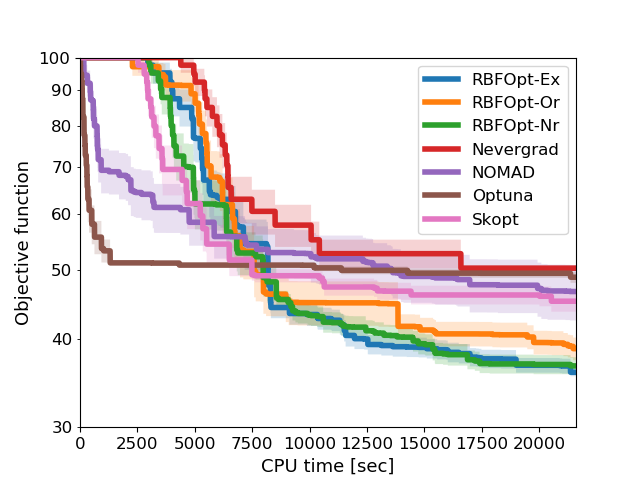}
  }\\
  \subcaptionbox{RBFOpt, extended space    \label{fig:rf_3}}{
    \includegraphics[width=0.48\textwidth]{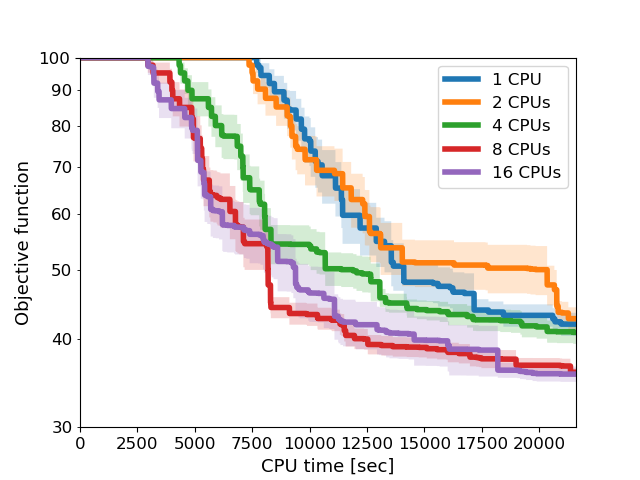}
  }
  \subcaptionbox{Comparison of several algorithms (adjusted time)    \label{fig:rf_4}}{
    \includegraphics[width=0.48\textwidth]{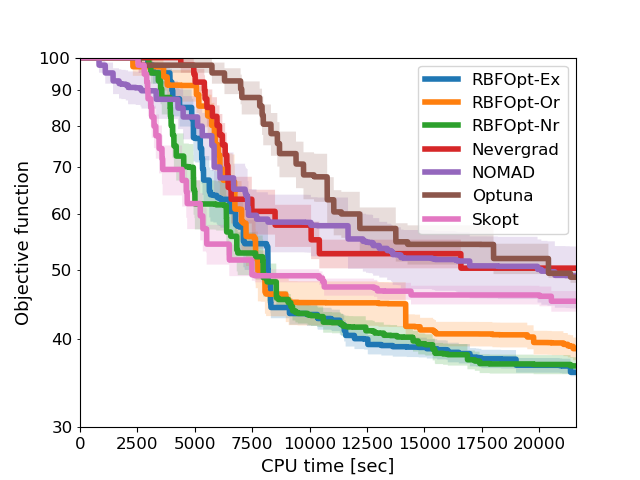}
  }
  
  \caption{Average objective function value over wall-clock time, for
    the optimization of the hyperparameters of a random forest
    classifier. The shaded regions represent the area containing the
    average plus/minus the standard error. In subfigure
    \ref{fig:rf_2}, the ``RBFOpt-Or'' and ``RBFOpt-Ex'' curves use the
    default settings of the algorithm in original and extended space,
    respectively; the ``RBFOpt-Nr'' version is like ``Extended'' but
    skips the Refinement step. In subfigure \ref{fig:rf_4}, we adjust
    the plots so that all algorithms report time in a similar way, see
    discussion in the main text. Subfigures \ref{fig:rf_2} and
    \ref{fig:rf_4} use 8 CPUs.}
  \label{fig:rf}
\end{figure}

We evaluate the performance of the optimization algorithm to optimize
the hyperparameters of a random forest classifier on a specific
dataset. We use the RandomForestClassifier class implemented in
Scikit-learn \cite{scikit-learn}, trained on the ``forest cover type''
dataset. This is a classification dataset with 581012 samples of
dimension 54, and 7 classes; see
\url{http://archive.ics.uci.edu/ml/datasets/Covertype} for more
information. The RandomForestClassifier has 10 hyperparameters, listed
in Table \ref{tab:rfparameters}; three of them are categorical, but
one of them has only two possible values (the ``criterion'' parameter)
and it is therefore treated as a binary variable. We evaluate the
performance of a classifier by 5-fold cross validation, and use the
average performance on the test set as the objective function. To
transform it into a minimization problem, we computed the objective
function as $100$ minus the recorded accuracy. The CPU time for a
single evaluation of the objective function varies a lot, depending on
the chosen hyperparameters; it is typically between $100$ and $1000$
seconds, but it can take up to a few hours. We use default values
for all parameters of RBFOpt. We compare several algorithms:
\begin{itemize}
\item RBFOpt with the extended space formulation of categorical
  variables;
\item RBFOpt with the original space formulation of categorical
  variables;
\item Nevergrad with the TwoPointsDE algorithm, which has the best
  performance for small $\tau$ in the experiments of Sect.~\ref{s:litcomp};
\item NOMAD, using the p-MADS parallel version;
\item Optuna;
\item Scikit-Optimize, using the {\tt gp} base estimator and one-hot
  encoding for categorical variables.
\end{itemize}
SMAC is excluded from this set of experiments due to technical
problems when running the available beta version on multiple CPUs. We
additionally tested the Coop-MADS variant of NOMAD, but we do not
report the corresponding results because p-MADS proved to be superior
on this problem instance. Note that the one-hot encoding employed in
Scikit-Optimize uses the same principle as the extended space
formulation of this paper. The wall-clock time limit is set to 6 hours
for all algorithms.

We run 20 different random seeds for each of these algorithms, using
these seeds to initialize the optimization algorithms and the training
of the classifier, thereby making the training deterministic and
reproducible. As remarked in Sect.~\ref{s:litcomp}, RBFOpt and NOMAD
assume that the objective function is deterministic, therefore fixing
the random seeds is justified for these algorithms. For the remaining
solvers, which target hyperparameter optimization problems, the random
seed does not have to be fixed and in principle their performance
could improve if we allow the solvers to evaluate the same point
multiple times with different random seeds for the training phase. In
particular, with our setup each algorithm observes only one
realization of the chosen generalization error estimator (i.e.,
accuracy using 5-fold cross validation) for a given values of the
hyperparameters, whereas hyperparameter optimization solvers can in
principle observe multiple realizations and use this information to
their advantage. We do not explore this possibility, noting that given
the relatively tight wall-clock time limit, we expect the approach
described above (i.e., fixed random seed for each run) to be a
reasonable trade off. We report the average objective function value
over time, where the average is taken with respect to the 20 random
seeds. Results are given in Fig.~\ref{fig:rf}. The runs corresponding
to the same random seed are initialized with the same set of points
for RBFOpt, Nevergrad and Scikit-Optimize: this has the goal of
reducing variance in the experiments. NOMAD is initialized using the
first point in the latin hypercube design generated by RBFOpt, while
Optuna uses its own initialization strategy.

In Fig.~\ref{fig:rf_1} and \ref{fig:rf_3}
we plot the objective function value for RBFOpt in original and
extended space, respectively, using up to 16 CPUs.  The plots showcase
the benefits of asynchronous parallel optimization when the main
concern is the wall-clock time, rather than overall efficiency of the
search in terms of the number of objective function evaluations. When
using multiple CPUs, not only we improve the objective function much
faster, but we eventually find better solutions on
average. Furthermore, the extended space formulation performs
noticeably better than the original space formulation, supporting our
conclusions from Sect.~\ref{s:expcat}.

In Fig.~\ref{fig:rf_2} we compare several variants of RBFOpt with
other solvers, using 8 CPUs. We choose 8 CPUs rather than 16 so that
the problem is still moderately difficult. The plots indicate that
RBFOpt, in all its variants, attains lower objective function values
than all the remaining solvers, and the corresponding curves are below
the other solvers for most of the time interval considered. The plots
for NOMAD and Optuna decrease much faster than the other algorithms at
the very beginning, but this is mostly due to the fact that both NOMAD
and Optuna report the result of each function evaluation immediately,
whereas the other algorithms only report results after evaluating a
first batch of 16 points using Python's {\tt
  multiprocessing.Pool}. This implies that for NOMAD and Optuna the
curves start improving after the first evaluation, whereas for all
other algorithms no improvement is reported until the first 16
function evaluations are completed. To put all algorithms on equal
footing, we plot the same data in Fig.~\ref{fig:rf_4}, with the
difference that we now assign to the first 15 function evaluations the
same time stamp as the 16th. In this fairer setting, NOMAD and Optuna
no longer enjoy an advantage in the early stages of the
optimization. We point out that in the recent paper
\cite{lakhmiri2019hypernomad}, NOMAD is shown to be more effective
than RBFOpt on problems with heavily constrained feasible regions: on
those problems, the local nature of NOMAD allows staying inside the
(black-box) constraints, while the global nature of RBFOpt leads to
the exploration of a large proprtion of infeasible points, even if
their violation is penalized. However, the problem considered in this
section is unconstrained, and RBFOpt improves the objective function
more quickly than NOMAD or any of the other solvers.

Summarizing, RBFOpt performs better than all other
tested algorithms on this hyperparameter optimization problem; among
RBFOpt variants, compared using 8 CPUs, the extended space formulation
has a clear advantage, and the Refinement step has little impact in
these experiments, even though it was highly beneficial in our
previous tests.

\section{Conclusion}
\label{s:conclusion}
Our extensive numerical evaluation indicates that many ingredients
contribute to the effectiveness of the optimization algorithm
implemented in RBFOpt. More precisely, this paper explores the impact
of different ways of modeling categorical variables (and discusses how
to deal with the ensuing difficulties), the use of RBF interpolants
without the unisolvence property, refinement search to locally improve
solutions, and asynchronous parallel optimization. Individually, each
of these components provides some benefit, with the refinement search
having the most noticeable impact in serial optimization. Together,
they contribute to a powerful optimization algorithm that appears to
be one of the most efficient derivative-free solvers available for
highly nonconvex problems, and that is capable of handling
unconstrained mixed-variable problems.

\bibliographystyle{plainurl} \bibliography{phd}

\end{document}